\documentclass{article}

\PassOptionsToPackage{numbers, compress}{natbib}
\usepackage[preprint]{neurips_2026}

\usepackage[utf8]{inputenc} % allow utf-8 input
\usepackage[T1]{fontenc}    % use 8-bit T1 fonts
\usepackage{hyperref}       % hyperlinks
\usepackage{url}            % simple URL typesetting
\usepackage{booktabs}       % professional-quality tables
\usepackage{amsfonts}       % blackboard math symbols
\usepackage{nicefrac}       % compact symbols for 1/2, etc.
\usepackage{microtype}      % microtypography
\usepackage{xcolor}         % colors
\usepackage{algorithm}
\usepackage{algorithmic}

\usepackage{amsmath}
\usepackage{amssymb}
\usepackage{mathtools}
\usepackage{amsthm}
\usepackage{thmtools}
\usepackage{thm-restate}

\usepackage{etoc}

\usepackage{caption}

\usepackage{enumitem}
\usepackage[capitalize, noabbrev]{cleveref}
\AddToHook{env/lemma/begin}{\crefalias{theorem}{lemma}}
\crefname{lemma}{Lemma}{Lemmas}

\AddToHook{env/proposition/begin}{\crefalias{theorem}{proposition}}
\crefname{proposition}{Proposition}{Propositions}

\AddToHook{env/remark/begin}{\crefalias{theorem}{remark}}
\crefname{remark}{Remark}{Remarks}

\AddToHook{env/corollary/begin}{\crefalias{theorem}{corollary}}
\crefname{corollary}{Corollary}{Corollaries}

\AddToHook{env/assumption/begin}{\crefalias{theorem}{assumption}}
\crefname{assumption}{Assumption}{Assumptions}

\AddToHook{env/proof/begin}{\crefalias{theorem}{proof}}
\crefname{proof}{Proof}{Proofs}

\Crefname{equation}{Eq.}{Eqs.}
\Crefname{figure}{Fig.}{Figs.}

\theoremstyle{plain}
\newtheorem{theorem}{Theorem}[section]
\newtheorem{proposition}[theorem]{Proposition}
\newtheorem{lemma}[theorem]{Lemma}
\newtheorem{corollary}[theorem]{Corollary}

\theoremstyle{definition}

\newtheorem{assumption}[theorem]{Assumption}

\theoremstyle{remark}
\newtheorem{remark}[theorem]{Remark}

\DeclareMathOperator{\err}{\textbf{err}}

\newcommand{\E}{\mathbb{E}}
\newcommand{\TV}{\mathbf{TV}}

\usepackage[disable,textsize=tiny]{todonotes}
\usepackage{comment}
\usepackage{dsfont}
\usepackage{xcolor}
\usepackage[skins]{tcolorbox}
\usepackage{soul}
\sethlcolor{purple}
\usepackage{multirow}
\usepackage{optidef}
\usepackage{siunitx}
\usepackage{wrapfig}
\usepackage[subtle]{savetrees}

\definecolor{NTGreen}{RGB}{1, 113, 4}
\definecolor{NTBlue}{RGB}{0, 77, 128}
\definecolor{NTPink}{RGB}{150, 15, 82}

\tcbset{commonstyle/.style={boxrule=0pt,sharp corners,enhanced jigsaw,nobeforeafter,boxsep=0pt,left=\fboxsep,right=\fboxsep}}
\newtcolorbox{mycolorbox}[1][]{commonstyle,#1}

\newcommand{\gtext}[1]{\textcolor{NTGreen}{#1}}

\newcommand{\ptext}[1]{\textcolor{NTPink}{#1}}
\newcommand{\tbf}[1]{\textbf{#1}}

\title{AdaptNC: Adaptive Nonconformity Scores for Conformal Prediction under Distribution Shift}
\author{%
    Renukanandan Tumu$^{1}$ \thanks{$^{1}$ All authors are from Department of Electrical and Systems Engineering, University of Pennsylvania, Philadelphia, Pennsylvania (email: \{nandant, singhadi, rahulm\}@seas.upenn.edu)}
    \And
    Aditya Singh$^{1}$ 
    \And
    Rahul Mangharam$^{1}$ 
}

\begin{document}
    \maketitle
    \etocdepthtag.toc{main}
    
\begin{abstract}
Rigorous uncertainty quantification is essential for the safe deployment of autonomous systems in dynamic environments. Conformal Prediction (CP) provides a distribution-free framework for this task, yet its standard formulations rely on exchangeability assumptions that are violated by the distribution shifts inherent in real-world robotics. Existing online CP methods maintain target coverage by adaptively scaling the conformal threshold, but typically employ a static nonconformity score function. We show that this fixed geometry leads to highly conservative, volume-inefficient prediction regions when environments undergo structural shifts. To address this, we propose \textbf{AdaptNC}, a framework for the joint online adaptation of both the nonconformity score parameters and the conformal threshold. AdaptNC leverages an adaptive reweighting scheme to optimize score functions, and introduces a replay buffer mechanism to mitigate the coverage instability that occurs during score transitions. We evaluate AdaptNC on diverse robotic benchmarks involving multi-agent policy changes, environmental changes and sensor degradation. Our results demonstrate that AdaptNC significantly reduces prediction region volume compared to state-of-the-art threshold-only baselines while maintaining target coverage levels.
\end{abstract}

% \begin{abstract}
%     Uncertainty quantification is a critical component for the deployment of real-world robotic systems.  Although conformal prediction provides a principled, distribution-free approach to this problem, real-world robotic systems typically operate in dynamic and evolving environments that induce distribution shifts, thereby violating the exchangeability assumptions required to guarantee nominal coverage. Previous online conformal prediction methods allow principled adaptation of the score threshold in settings with distribution shift, but do not adapt the non-conformity score and can yield conservative regions, which impact the performance of robotic systems. We propose AdaptNC, a method that adapts the non-conformity score function and score threshold via optimization-based score function design and a replay buffer to mitigate coverage shock. Demonstrations show that our method can adapt to structural distribution shifts caused by policy changes in multi-agent settings, environmental changes, and system degradation.
% \end{abstract}

    \section{Introduction}
\vspace{-0.5em}
Autonomous systems operating in real-world environments are subject to constant change. These changes induce distribution shift, driven by changing environmental conditions, such wind or weather, system degradation due to damage or wear and tear, or through multi-agent interactions, which can all degrade the performance of predictive models. 
For safety-critical applications such as autonomous driving and robotic navigation, identifying these shifts and providing robust, uncertainty-aware predictions is not just a performance requirement; it is a safety necessity \cite{brunke_safe_2022}. 
Distribution shift in robotics does not only change the magnitude or frequency of errors, it can change the structure of the errors. For example, tire damage in a ground vehicle or strong winds affecting a drone can induce directionally biased prediction errors. Because prediction regions are often used to define where robots should slow down, avoid, or plan conservatively \cite{11274485}, these structural changes directly affect the usefulness of uncertainty estimates. Thus, prediction regions should not only adapt in size, but also in shape, as environmental and system conditions evolve.

Uncertainty quantification has been a focus of much recent work, and conformal prediction (CP) \cite{vovk_algorithmic_2022, lei_distribution-free_2018}, has emerged as a popular method for its flexibility to be applied to predictors of all types, its statistical grounding, and its ease of use for the practitioner. CP constructs prediction sets which contain the true label with high ($1 -\alpha$) probability, under the assumption of exchangeability between the calibration and test distributions. In CP, the threshold controls the size of the prediction region, while the non-conformity score function controls its geometry. Online adaptation has largely focused on the former, leaving the latter fixed even when the structure of the data changes.

Recent online CP methods adapt the conformal threshold to maintain coverage under distribution shift\cite{gibbs_adaptive_2021,gibbs_conformal_2024, zaffran_adaptive_2022}. However, because they keep the non-conformity score fixed, they can only resize prediction regions rather than reshape them. This limitation is important in robotic settings, where shifts such as sensor degradation or policy changes can alter the geometry of prediction errors. Separately, previous work has addressed the design of non-conformity score functions\cite{kiyani_length_2024,gao_volume_2025}, including control and robotics settings with multi-dimensional outputs \cite{tumu_multi-modal_2024,tumu_physics_2023,cleaveland_conformal_2024,braun_minimum_2025}, however, this work does not address the adaptation of the non-conformity score in response to distribution shift.

Unlike existing methods that only tune the threshold, AdaptNC optimizes the underlying score function to ensure that prediction regions remain tight and informative even as the data distribution evolves. AdaptNC also shifts the conformal threshold. In order to provide principled score function adaptation, we introduce an \textit{adaptive reweighting scheme} that leverages DtACI expert weights to prioritize the most relevant historical observations during score optimization.
The joint adaptation of the score function and threshold introduces the phenomenon of ``coverage shock": the errors in coverage caused by evaluating old threshold values under a new score function.  To mitigate ``coverage shock''  we employ a counterfactual \textit{replay mechanism} to recalibrate the conformal threshold following a score update. Across three robotic settings with structural distribution shifts, AdaptNC maintains target coverage while producing smaller prediction regions than threshold-adaptation baselines, demonstrating that score adaptation improves efficiency without sacrificing empirical coverage.

\textbf{Contributions:} This paper addresses the challenge of constructing efficient, uncertainty-aware prediction regions under distribution shift by contributing:
\vspace{-0.5em}
\begin{enumerate}
    \item \textbf{Online Score Function and Threshold Adaptation:} We present a method for the online optimization of non-conformity scores and their corresponding conformal thresholds that minimizes prediction region volume while enabling adaptation to shifting distributions.
    
    \item \textbf{Dynamic Weighting and Replay Buffer:} We develop a weighting mechanism for data based on the rate of distribution shift, and a replay mechanism to maintain coverage stability during score parameter transitions.
    
    \item \textbf{Empirical Validation on Robotic Benchmarks:} We evaluate AdaptNC on robotic tasks involving structural distribution shifts, including multi-agent policy changes and sensor degradation. We demonstrate that AdaptNC reduces prediction volume compared to state-of-the-art online CP baselines while maintaining target coverage levels.
    
\end{enumerate}

Furthermore, we emphasize that AdaptNC is not a direct composition of prior adaptive conformal prediction and score-adaptation methods. Instead, it formulates a new setting involving coupled adaptation, characterizes the instabilities arising from a naive integration, and introduces the mechanisms necessary to ensure stable and effective joint adaptation. Additional discussion is provided in Appendix~\ref{app:direct_comb}.

    \vspace{-0.5em}
    \begin{figure*}[t]
        \centering
        \includegraphics[width=0.7\textwidth]{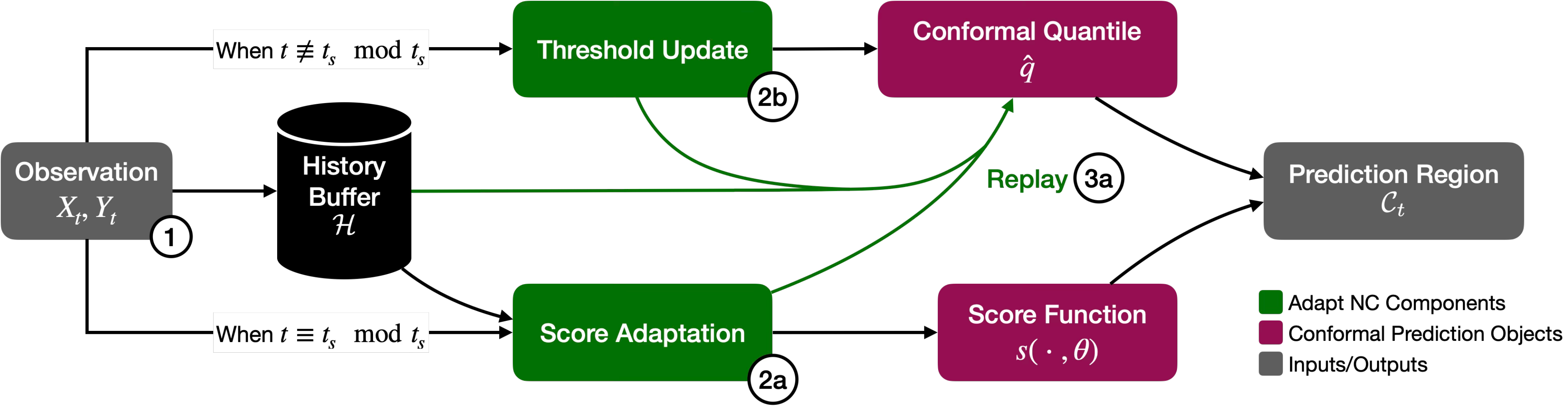}
        \captionsetup{font=small}
        \caption[Overview of the AdaptNC Framework]{The AdaptNC framework is an online algorithm designed to handle distribution shifts. The above figure shows the procedure at a timestep $t$. (1) The observation for the current timestep arrives, and is added to the history buffer. Every $t_s$ timesteps, the \ptext{score function} is adapted by \gtext{weighted score adaptation} (2a). This adaptation yields a score function \ptext{$s(\cdot,\theta_t)$}, the \gtext{replay} algorithm (3a) compensates for the distribution shift generated by the change in score function by replaying recent samples from the history buffer $\mathcal{H}$ and generates the conformal quantile \ptext{$\hat{q}$}. When the score is not being adapted, only the \ptext{conformal quantile} is adapted through the \gtext{threshold update} step.}
        \label{fig:adaptnc-overview-figure}
        \vspace{-1.5em}
    \end{figure*}
    \section{Problem Setup and Preliminaries}
\vspace{-0.5em}
\subsection{Online Prediction Setting}
\vspace{-0.3em}
We consider an online $\tau$-step prediction setting. At each time step $t = 1, \ldots, T$, the predictor observes features $X_{t}\in \mathcal{X}$ and produces a $\tau$-step prediction $\hat{Y}_{t} = h(X_t) = \{\hat{y}_{t+1}, \ldots, \hat{y}_{t+\tau}\} \in \mathcal{Y}^{\tau}$. The future outcome $Y_t \in \mathcal{Y}^{\tau}$ is then observed. The joint distribution of $(X_t,Y_t)$ may change over time.

We seek to build a prediction region $\mathcal{C}_t$ that contains the future outcome $Y_t$ with probability $1-\alpha$, where $\alpha\in(0,1)$ is the target miscoverage rate. The prediction regions are constructed using a non-conformity score function $s(X_{t}, y;\theta_t)$ where $\theta_t$ are parameters of the score function at time $t$, as well as a conformal threshold, $q_t$. We define $q_{\mathcal{D}_t}(\delta)$ to be the empirical $\delta\in [0,1]$ quantile of the score dataset $\mathcal{D}_t$, and $q_t = q_{\mathcal{D}_t}(\delta_t)$. The prediction region is the set of points with a non-conformity score less than the conformal quantile: $\mathcal{C}_{t}(X_{t};\theta_t,q_t) \coloneqq \{y \in \mathcal{Y}^{\tau} : s(X_{t}, y;\theta_t) \leq q_t\}$. The threshold $q_t$ controls the size of the prediction region, while the score parameters $\theta_t$ control its geometry.
\vspace{-0.8em}
\paragraph{Desiderata:}
We pursue the objectives of coverage that approaches the target, and efficiency of the prediction region under settings with distribution shift.

\textbf{Coverage:}
Under distribution shift, exact finite-sample coverage at every time step is generally not attainable. Instead, we seek the asymptotic miscoverage rate $\alpha$, in settings where the magnitude of distribution shift decays over time. We define the error indicator $\err_{t}= \mathds{1}\{Y_{t}\notin \mathcal{C}_{t}(X_{t};\theta_t,q_t) \}$, and we target:
\vspace{-0.2em}
\begin{equation}\label{eq:long-term-coverage}
    \lim_{T \rightarrow \infty}\frac{1}{T}\sum_{t=1}^{T}\err_{t}= \alpha.
\end{equation}\vspace{-0.7em}

\textbf{Efficiency:}
We would also like to ensure that the prediction regions $\mathcal{C}_{t}$ are small and informative to downstream algorithms. Subject to our coverage desideratum, our objective is to minimize the volume of the prediction set $\mathcal{C}_{t}$. We formulate the optimization problem:
\begin{subequations}\label{eq:efficiency-optim}
\begin{align}
    \min_{\theta_t}\quad & \text{Vol}(\mathcal{C}_{t}(X_{t+1};\theta_t,q_t)), \quad  
    \text{s.t.} ~~\Pr[Y_{t+1} \in \mathcal{C}_{t}(X_{t+1};\theta_t,q_t)] \geq 1-\alpha.
\end{align}
\end{subequations}
In the online setting, the future distribution is unknown, so this problem serves as an idealized objective: AdaptNC seeks to reduce prediction-region volume while maintaining long-run coverage.
\vspace{-0.3em}
\subsection{Preliminaries}
\vspace{-0.3em}
\textbf{Conformal Prediction:}
Conformal prediction (CP) and its split variants provide distribution-free, finite-sample prediction regions which guarantee ground truth coverage with $1-\alpha$ probability under the assumption of exchangeable calibration and test data \cite{vovk_algorithmic_2022, lei_distribution-free_2018, barber_predictive_2020}. Under distribution shift, this exchangeability assumption may fail, motivating online methods that adapt the threshold over time.

\textbf{Online Threshold Adaptation:}
Online conformal prediction methods adapt the threshold to maintain long-run coverage in settings where distribution shift occurs.
DtACI \cite{gibbs_conformal_2024} targets long-run coverage under arbitrary distribution shifts by adaptively reweighting a set of candidate adaptation rates based on their historical performance. It obtains a long-term coverage guarantee of the form in \cref{eq:long-term-coverage} in settings where the size of the distribution shift decays over time.

The DtACI algorithm seeks to compensate for distribution shift by changing the adaptive miscoverage level $\alpha_t$ used to select the score threshold $q_t=q_{\mathcal{D}_t}(1-\alpha_t)$. The algorithm seeks to find $\alpha_t$ that minimizes the regret to the hindsight optimal quantile $\alpha^*_t$. This difference is bounded in \cref{thm:dtaci_regret}. Observations in the algorithm are quantiles $\beta_t$, which are the lowest empirical quantiles of the dataset $\mathcal{D}_t$ that would have covered the true outcome at time $t$. Formally,  
\begin{equation} \label{eq:def_beta_t}
        \beta_t \coloneqq \inf\left\{\beta: Y_t \in \{y : s(X_t, y;\theta_t) \leq \hat{q}_{\mathcal{D}_{t}}(\beta)\}\right\}
\end{equation}
To achieve this, DtACI maintains a set of $k$ candidate adaptation rates $\{\gamma_i\}_{i=1}^k$ and corresponding weights $\{w_t^i\}_{i=1}^k$. At each time step, the algorithm performs an exponential weights update on the weights based on the error incurred by each candidate rate. An averaged quantile $\bar{\alpha}_t$ is then computed using the weights, and the quantile threshold $q_{\mathcal{D}_t}(1-\bar{\alpha}_t)$ is updated accordingly. The algorithm is presented in \cref{alg:dtaci} (in Appendix~\ref{appendix:method}).

    \vspace{-0.7em}
    \section{Related Work}
% Conformal prediction has been increasingly used for applications in safety-critical systems, such as autonomous driving and robotic control \cite{11274485,lindemann_safe_2023}. In these settings, prediction regions directly influence downstream planning and control decisions, so uncertainty estimates must remain valid and informative under changing conditions. Efficient and adaptive conformal prediction methods can enhance the reliability of these systems by providing tighter uncertainty bounds that reflect the current data distribution. \citet{dixit_adaptive_2023} demonstrate the utility of adaptive conformal prediction in autonomous driving scenarios, providing a framework for utilizing adaptive conformal methods in control. These applications motivate adaptive prediction regions that remain both valid and efficient under distribution shift. However, existing adaptive CP methods for these settings typically rely on fixed non-conformity scores or threshold-only adaptation.
% \vspace{-0.8em}
\vspace{-0.5em}
\paragraph{Online Threshold Adaptation.}
Online conformal prediction methods address distribution shift primarily by adapting the conformal threshold or calibration quantile over time. Adaptive Conformal Inference (ACI) updates the quantile based on recent coverage errors \cite{gibbs_adaptive_2021}, while related methods such as AgACI and multivalid conformal prediction provide alternative mechanisms for maintaining coverage under non-stationarity \cite{zaffran_adaptive_2022,bastani_practical_2022}. DtACI further advances this line of work by adaptively weighting multiple candidate adaptation rates and providing long-term coverage regret guarantees \cite{gibbs_conformal_2024}. However, these online CP methods adapt the threshold while leaving the non-conformity score fixed. As a result, they can resize prediction regions under distribution shift, but they cannot reshape them when the geometry of prediction errors changes. Other approaches address distribution shift by reweighting calibration points to better match the test distribution under covariate shift~\cite{tibshirani_conformal_2019,barber_conformal_2023}. These methods modify the effective calibration distribution, and have been used in reinforcement learning settings \cite{taufiq_conformal_2022, kuipers_conformal_2024} but they do not adapt the non-conformity score online. 
\vspace{-0.8em}
\paragraph{Efficient and Shape-Adaptive Non-Conformity Scores.}
A separate line of work studies how the choice of non-conformity score affects the shape and volume of conformal prediction regions. Recent methods design volume-efficient scores for general prediction settings~\cite{kiyani_length_2024,gao_volume_2025}, as well as for control and robotics problems with multi-dimensional outputs~\cite{tumu_multi-modal_2024,tumu_physics_2023,cleaveland_conformal_2024,braun_minimum_2025}. These methods show that intentional design of the non-conformity score function can substantially improve prediction-region efficiency and downstream usefulness. However, they typically design or optimize the score for a fixed data-generating process. They do not consider online score adaptation under distribution shift, where the geometry of prediction errors may change over time. 

    \vspace{-0.7em}
    \section{AdaptNC Method}
\vspace{-0.5em}
Our method, Adaptation of Non-Conformity Scores (AdaptNC), extends online conformal prediction by jointly adapting the conformal threshold and non-conformity score function (NCSF). (1) \textbf{Threshold adaptation} maintains long-run coverage, (2) \textbf{Score Parameter Adaptation} improves efficiency under distribution shift, and (3) \textbf{Replay} recalibrates the threshold after score updates to mitigate coverage shock. At each timestep, the conformal threshold is updated using the quantile $q_{\mathcal{D}_t}(1-\bar{\alpha}_t)$ based on the current NCSF and newly observed data. Score parameters $\theta_t$ are periodically updated using estimates of distribution shift to obtain more efficient regions, and replay rescores recent observations under the updated NCSF ccount for the induced change in the score distribution . Due to computational intensity, steps (2) and (3) are performed every $t_s$ timesteps, while (1) is updated at every timestep. An overview is shown in \cref{fig:adaptnc-overview-figure}, with full details in \cref{alg:adaptnc} (Appendix~\ref{appendix:method}).

% Our method, Adaptation of Non-Conformity scores (AdaptNC) extends online conformal prediction by jointly adapting the conformal threshold and non-conformity score function (NCSF). (1) \textbf{Conformal threshold adaptation} maintains long-run coverage, (2) \textbf{Score Parameter Adaptation} improves prediction region efficiency under distribution shift, and (3) \textbf{Replay} recalibrates the threshold after score updates to mitigate coverage shock. In the conformal threshold update step, we chose the quantile $q_{\mathcal{D}_t}(1-\bar{\alpha}_t)$ based on the current NCSF and newly observed data. In the score parameter adaptation step, we use the current estimates of the rate of change of the data distribution to fit efficient non-conformity score parameters $\theta_t$. Finally, in the replay step, we account for the change induced in the score distribution by the change of the NCSF by replaying recent observations. Due to computational intensity, steps (2) and (3) are only run every $t_s$ timesteps, while (1) is run every timestep. The AdaptNC algorithm is presented pictorially in \cref{fig:adaptnc-overview-figure}, and in detail in \cref{alg:adaptnc} (in Appendix~\ref{appendix:method}).   

The data observed is saved as tuples $(X_t,Y_t,\hat{Y}_t)$ in a history buffer $\mathcal{H}$. For each observation, $\beta_t$ is calculated according to \cref{eq:def_beta_t}. The dataset $\mathcal{D}_t$ used in the quantile calculation is a rolling window of the last $W$ timesteps.

% \subsection{Overview}
% What state AdaptNC maintains.
% What happens each timestep.
% What happens every t_s timesteps.
% Three components and their roles.

% \subsection{Online Threshold Adaptation}
% Uses DtACI for current score function.
% Computes beta_t.
% Updates expert weights.
% Produces q_t.
% Expert weights are reused for score adaptation.

% \subsection{Adaptive Score Optimization}
% Motivation: threshold adaptation resizes; score adaptation reshapes.
% Reweight history with DtACI weights/rates.
% Solve weighted efficiency objective.

% \paragraph{Adaptive history weighting.}
% Define omega_{j,t}.

% \paragraph{Volume-efficient score fitting.}
% KDE high-density region.
% Convex hull / score parameterization.

% \subsection{Replay for Score Updates}
% Define coverage shock.
% Explain why changing theta disrupts DtACI state.
% Counterfactual replay over last W observations.

% AdaptNC has four steps: (1) \textbf{Observation}, (2) \textbf{Conformal threshold update}, (3) \textbf{Score parameter adaptation}, and (4) \textbf{Replay}. In the observation step, we observe the feature vector $X_t$, produce a point prediction $\hat{Y}_t$, and then observe the true outcome $Y_t$. We compute the non-conformity score $s(X_t, Y_t; \theta_t)$ using the current score parameters $\theta_t$ and store the tuple $(X_t, Y_t, s(X_t, Y_t; \theta_t))$ in a history buffer $\mathcal{H}$.
\vspace{-0.8em}
\paragraph{Conformal Threshold Update}
At every timestep, AdaptNC uses DtACI (\cref{alg:dtaci}) to adapt the conformal threshold online. The goal of this step is to select an adaptive miscoverage level $\bar{\alpha}_t$ and the corresponding quantile $q_t = q_{\mathcal{D}_t}(1-\bar{\alpha}_t)$, where $\mathcal{D}_t$ is a rolling window of the last $W$ timesteps. To enable coverage under unknown distribution shift, $k$ ACI \citep{gibbs_adaptive_2021} experts are maintained and dynamically reweighted. The procedure involves calculating an inverse quantile to transform the raw non-conformity score, expert reweighting, and finally expert update. This procedure aims to estimate $\alpha^*_t$, the optimal value for the miscoverage rate that would yield the correct coverage.

Based on the error between the predicted quantile for each expert $\alpha_t^i$ and the observed $\beta_t$, the weight $w_t^i$ is exponentially reweighted. When these weights are applied to the experts, we get the algorithm's estimate of the current conformal quantile $\bar{\alpha}_t$. Finally, each expert is updated according to the expert learning rate $\gamma_i$ and the error indicator $\err_t^i$, which is $1$ if the observation is not covered by the $i$-th expert, and $0$ otherwise. The complete DtACI algorithm is presented in \cref{alg:dtaci} (in Appendix~\ref{appendix:method}).

%Needs Definition: \beta, \alpha, \alpha*, \sigma
\vspace{-0.8em}
\paragraph{Score Parameter Adaptation}
Threshold adaptation can maintain coverage under distribution shift, but it can only resize prediction regions for a fixed score function. To adapt the geometry of the prediction region, AdaptNC periodically updates the score parameters $\theta_t$. The objective of this step is to find score parameters that reduce prediction-region volume while preserving the desired coverage behavior. Since the future distribution is unknown, AdaptNC approximates this objective using a relevance-weighted history buffer, where the relevance weights are derived from the current DtACI expert weights $w_t^i$ and adaptation rates $\gamma_i$.

Let $\mathcal{H}_t$ denote the history buffer at time $t$. For each historical observation indexed by $j\leq t$, AdaptNC assigns a relevance weight using the DtACI expert mixture:
\begin{equation}\label{eq:weight_calc}
    \bar{\omega}_{j,t}
    =
    \sum_{i=1}^{k} w_t^i(1-\gamma_i)^{t-j},
    \qquad
    \omega_{j,t}
    =
    \frac{\bar{\omega}_{j,t}}
    {\sum_{\ell\in\mathcal{H}_t}\bar{\omega}_{\ell,t}}
\end{equation}
This weighting scheme induces a weighted empirical distribution $\mathcal{H}_t^\omega$. Experts with larger adaptation rates $\gamma_i$ emphasize recent observations, while experts with smaller adaptation rates retain longer memory. Therefore, the DtACI mixture provides an adaptive estimate of how much historical data should influence score adaptation.

Using the weighted history distribution $\mathcal{H}_t^\omega$, AdaptNC seeks score parameters that minimize prediction-region volume subject to weighted empirical coverage. This gives the objective
\begin{subequations}\label{eq:actual-volume-problem}
\begin{align}
    \min_{\theta}\quad
    \mathbb{E}_{X\sim \mathcal{H}_t^\omega}
    \left[
        \mathrm{Vol}\left(\mathcal{C}(X;\theta,q_\theta)\right)
    \right] \quad
    \text{s.t.}~~\mathbb{P}_{(X,Y)\sim \mathcal{H}_t^\omega}
    \left[
        Y\in \mathcal{C}(X;\theta,q_\theta)
    \right]
    \geq 1-\alpha,
\end{align}
\end{subequations}
where $q_\theta$ denotes the threshold induced by the score function with parameters $\theta$ on the weighted history. In practice, AdaptNC approximates this objective by estimating a high-density region of the weighted residual distribution and fitting a convex score template to that region.

To estimate the high-density region, AdaptNC forms residual samples $Y_j-h(X_j)$ from the weighted history buffer and constructs a weighted kernel density estimate (KDE)~\citep{parzen_estimation_1962}. We then sample candidate residuals from the KDE, evaluate their estimated density, and retain the highest-density samples whose empirical mass is approximately $1-\alpha$. We denote this retained set by $\hat{R}_t$. This Monte Carlo procedure provides an empirical approximation of the high-density residual region under the current weighted data distribution; the full algorithmic details are given in Appendix~\ref{sec:mc-kde}.

Given the estimated high-density region $\hat{R}_t$, AdaptNC fits a convex shape template to cover these points, in a manner similar to \citep{tumu_multi-modal_2024}. We parameterize the template by $\theta_t=\{A\in\mathbb{R}^{r\times p}, b\in\mathbb{R}^r\}$, where $r$ is the number of facets. The resulting non-conformity score is: $s(X,y;\theta_t) = \max_{j\in\{1,\ldots,r\}}\left(A_j(y-h(X))-b_j\right).$

Under this parameterization, the zero sublevel set of the score corresponds to the fitted convex template centered at the predictor output $h(X)$. We fit the convex hull of $\hat{R}_t$ using the QuickHull algorithm \cite{barber_quickhull_1996}. This produces convex prediction-region geometry, which is useful for downstream planning and control applications \citep{lindqvist_nonlinear_2020}.
\vspace{-0.8em}
\paragraph{Replay}\label{sec:method-replay}
Changing the score parameters alters the distribution of non-conformity scores. Consequently, the DtACI state prior to an update may remain calibrated to the previous score function. This mismatch induces \emph{coverage shock}, a transient calibration error immediately following a score update. In the DtACI regret bound, this appears as a potentially large change in the hindsight-optimal miscoverage level $\alpha_t^*$ across consecutive timesteps (see \cref{sec:theory-replay,appendix:gmm} for theoretical and empirical evidence ).

To mitigate coverage shock, AdaptNC replays the most recent $W$ observations after each score update. The replay procedure rescans these observations under the updated score parameters $\theta_t$, recomputes the corresponding quantiles $\beta_t$, and reruns the threshold updates. This can be interpreted as a counterfactual reconstruction of the DtACI state under the new score function. The resulting miscoverage level and threshold are then used for subsequent online predictions, effectively resetting the the threshold adaptation process under the new score distribution and reducing the abrupt calibration error induced by score updates.

% Changing the score parameters changes the distribution of non-conformity scores. Consequently, the DtACI state before a score update may be calibrated to the old score function rather than the updated one. This mismatch can produce \emph{coverage shock}: a transient calibration error immediately after a score update. In the DtACI regret bound, this effect appears as a potentially large change in the hindsight-optimal miscoverage level $\alpha_t^*$ across consecutive timesteps. Theoretical and empirical evidence for this phenomenon is presented in \cref{sec:theory-replay} and \cref{fig:adapt_nc_alpha_star}.

% To mitigate coverage shock, AdaptNC replays the most recent $W$ observations after each score update. Starting from a freshly initialized DtACI state, the replay procedure rescans these observations using the updated score parameters $\theta_t$, recomputes the corresponding realized quantiles $\beta_t$, and reruns the threshold-adaptation updates. This can be interpreted as a counterfactual evaluation of what the DtACI state would have been if the recent observations had all been scored under the updated score function. The resulting adaptive miscoverage level and threshold are then used for subsequent online predictions. Thus, replay resets the threshold-adaptation chain under the new score distribution, reducing the abrupt calibration error caused by score-function updates.

    \vspace{-0.7em}
    \section{Theoretical Analysis}
\vspace{-0.5em}
AdaptNC modifies adaptive conformal prediction in two ways: it updates the
nonconformity score online, and it replays recent observations after score
updates to realign the ACI expert states with the updated score. These two
mechanisms serve different goals. Score adaptation controls the geometry and
efficiency of the prediction region, while the ACI expert updates control
long-run coverage. Our theoretical analysis therefore separates these effects:
we first study long-run stability of prediction-region size, then show that
single-expert ACI coverage is preserved whenever replay induces only a
vanishing average perturbation of the expert recursion.
\(\theta_t\).
\vspace{-0.3em}
\subsection{Long-Run Score Function and Region Stability}
\vspace{-0.3em}
We make a series of assumptions in the long-run stability and coverage analysis, which are analogous to the long-run regimes considered in prior work \cite{gibbs_conformal_2024}. 
\begin{assumption}[Weighted Residual Stability]
    \label{assump:weighted_residual_stability}
    We assume that the weighted distribution of residuals \(\mathcal Z_t\) 
converges in $f$-divergence to the steady state $\mathcal{Z}_*$ in the asymptote: $D_\phi(\mathcal{Z}_t,\mathcal{Z}_{*}) \to 0$ as $t \to \infty$. We additionally assume that the MCKDE parameters $M,N \to \infty$. 
\end{assumption}
In AdaptNC, because the score function is fit from weighted residuals, the corresponding stability object is the weighted residual distribution used by the score-adaptation step. We require the regularity conditions for the HDR stability result in \cref{assump:hdr_regular}. We also assume that the MCKDE procedure converges in \cref{assump:template_consistency}. The proof of \cref{thm:score_func_stability} can be found in \cref{proof:score-func-stability}.

\begin{restatable}[Score Function Stability]{theorem}{scorefuncstability}\label{thm:score_func_stability}
Suppose \cref{assump:weighted_residual_stability,assump:hdr_regular,assump:template_consistency} hold. Let
$\widehat R_m$ be the MCKDE high-density residual region estimated at
score-update time $t_m$, and define\( K_m := \operatorname{conv}(\widehat R_m).\) Let\( R_\star := \{z:f_\star(z)\ge\tau_\star\}\) and \( K_\star := \operatorname{conv}(R_\star).\)
Let $s_{K_m}$ and $s_{K_\star}$ denote the canonical support-function scores
induced by $K_m$ and $K_\star$. Then, for any evaluation domain
$\mathcal A\subseteq\mathcal X\times\mathcal Y$,
\[
    \sup_{(x,y)\in\mathcal A}|s_{K_m}(x,y)-s_{K_\star}(x,y)| \xrightarrow{p}0 \quad\therefore\quad \sup_{(x,y)\in\mathcal A}|s_{K_{m+1}}(x,y)-s_{K_m}(x,y)| \xrightarrow{p}0.
\]
\end{restatable}
\begin{corollary}[Stability of prediction-region size]
\label{cor:prediction_region_size_stability}
Suppose
\(
    \sup_{(x,y)\in\mathcal A}
    |s_t(x,y)-s_\star(x,y)|
    \le \delta_t
\)
with \(\delta_t\to0\) in probability due to \cref{thm:score_func_stability}.
Under Assumption~\ref{assump:hdr_regular}, for every fixed threshold \(q\),
\[
    \left|
    \operatorname{Vol}(C_t(x;q))
    -
    \operatorname{Vol}(C_\star(x;q))
    \right|
    \le
    L_{\mathrm{vol}}\delta_t \quad\textrm{in probability}.
\]
Hence the size of the prediction region stabilizes as the score function
stabilizes in probability.
\end{corollary}

Thus, the score-stability result implies that the geometry and volume of the
prediction regions stabilize around the limiting convex-template region. In this
sense, the score-adaptation analysis is an efficiency result: once the residual
geometry stabilizes, AdaptNC no longer incurs persistent changes in prediction
region size due to score updates.
\vspace{-0.3em}
\subsection{Coverage Guarantees}\label{sec:theory-coverage}
\vspace{-0.3em}
The long-term coverage guarantee of our method follows from a variation on the per-expert guarantee of ACI, which we present in \cref{thm:per_expert_longrun_coverage}. This modification is necessary because the recursion used to prove the original ACI guarantee is perturbed by the score updates and replay mechanism. We show that the single-expert ACI guarantee is preserved whenever replay induces only a vanishing average perturbation of the expert recursion. We then give a sufficient condition for this perturbation to vanish: successive score updates must become small on replayed observations, and replayed observations must not concentrate near the conformal boundary.

\begin{theorem}[Per-Expert Coverage Guarantee with AdaptNC]
\label{thm:per_expert_longrun_coverage}
Fix an expert $i\in\{1,\ldots,k\}$ with step size $\gamma_i>0$. At time $t$, let the expert's prediction set be
\(
    \mathcal{C}_t^i = \mathcal{C}_t(X_t;\theta_t,q_{\mathcal{D}_t}(1-\alpha_t^i)),
\)
and define its realized miscoverage indicator
\(
    \err_t^i = \mathds{1}\{Y_t\notin \mathcal{C}_t^i\}.
\)
Suppose the expert state evolves according to the perturbed ACI recursion
\(
    \alpha_{t+1}^i = \alpha_t^i + \gamma_i(\alpha-\err_t^i) + \Delta_t^i,
\)
where $\Delta_t^i$ is the perturbation caused by score updates and replay. Assume that
the expert states remain uniformly bounded in the range $(0,1)$, then:
\begin{equation}
\label{eq:min-pert-to-lt-error}
    \frac{1}{T}\sum_{t=1}^T |\Delta_t^i| \longrightarrow 0 \implies \frac{1}{T}\sum_{t=1}^T \err_t^i \longrightarrow \alpha
\end{equation}
\end{theorem}

\begin{theorem}[Averaged Replay Perturbation Bound]
\label{thm:perturbation_to_zero_avg}
Fix an ACI expert \(i\) with step size \(\gamma_i>0\).
Let
\(\{t_m\}_{m\ge 1}\) be the score-update times, and let
\(\mathcal R_m\) be the replay window used at update \(t_m\). Let
\(s_m^-\) and \(s_m^+\) denote the score functions immediately before and
after the update. Define
\(
    \varepsilon_m := \sup_{r\in\mathcal R_m}\left|s_m^+(X_r,Y_r)-s_m^-(X_r,Y_r)\right|.
\)
Let \(q_{m,r}^{i,-}\) and \(q_{m,r}^{i,+}\) be the conformal thresholds used
by expert \(i\) at replay step \(r\) under the pre-update and post-update
scores. Suppose there exists \(L_q<\infty\) such that
\(
    |q_{m,r}^{i,+}-q_{m,r}^{i,-}| \le L_q\varepsilon_m
\)
for all \(m\) and \(r\in\mathcal R_m\). Suppose further that
\[
    \frac1T
    \sum_{m:t_m\le T}
    \sum_{r\in\mathcal R_m}
    \mathbf 1
    \left\{
        \left|
        s_m^-(X_r,Y_r)-q_{m,r}^{i,-}
        \right|
        \le
        (1+L_q)\varepsilon_m
    \right\}
    \longrightarrow 0.
\]
Then the replay-induced perturbations satisfy
\(
    \frac1T\sum_{t=1}^T|\Delta_t^i| \longrightarrow 0.
\)
% Therefore, if the expert states remain uniformly bounded, then
% \[
%     \frac1T\sum_{t=1}^T\err_t^i \longrightarrow \alpha.
% \]
\end{theorem}
The second supposition requires that replayed observations do not accumulate in the shrinking band around the conformal boundary where a small score update can change the miscoverage indicator. Score stability makes the width of this band shrink, but it does not by itself rule out many replayed points lying near the boundary. The assumption can be interpreted as a regularity condition on the data distribution. Taken together, the preceding results establish the single-expert long-run coverage property underlying AdaptNC. Proofs are provided in \cref{appendix:theory}. Since the DtACI long-run coverage guarantee holds for arbitrary sequences, AdaptNC inherits this guarantee for any sequence of score functions.

% We formally define the coverage shock to be the $\alpha^*$ shift term from \cref{rem:big-o-regret}:
% \begin{equation}
%     \textrm{coverage shock} = \left|\alpha_t^{*}-\alpha_{t-1}^{*}\right|.
% \end{equation}
% This, combined with a term that varies inversely with the size of the window $|W|$, combine to dominate the regret expression. Minimizing the coverage shock, or minimizing the difference between the optimal miscoverage level across consecutive timesteps will improve the regret performance of the algorithm. Conversely, the changes to the NCSF can have an outsized impact on this difference. The counterfactual replay mechanism is designed to mitigate this. Here, we present an intuitive explanation paired with evidence from a toy experiment in support.

% The hindsight optimal level of $\alpha*_t$ is related to the distribution of the non-conformity scores at time $t$. We decompose the shift in the score distribution into that caused by the observation of a new data point, and then the change in the NCSF. First, we note that the exchange of a new and old sample produces a bounded change in the score distribution in \cref{prop:dset-change-bound}. Next, we show that the distance caused by the change in NCSF parameter is not bounded by the data-distribution shift.
\vspace{-0.3em}
\subsection{Replay}\label{sec:theory-replay}
% \begin{wrapfigure}{r}{0.38\columnwidth}
%     \centering
%     \includegraphics[width=\linewidth]{media/results/gmm_alpha_star_d_shift.pdf}
%     \vspace{-5pt}
%     \caption{This figure shows the change in the optimal threshold $\alpha_t^*$ based on scores from the initial and final distributions, $\mathcal{N}_1$ and $\mathcal{N}_2$. The bottom pane shows the rate of distribution change, while the top shows the difference between optimal quantiles.}
%     % \vspace{-3em}
%     \label{fig:adapt_nc_alpha_star}
% \end{wrapfigure}
\vspace{-0.3em}
We formally define \emph{coverage shock} as the $\alpha^*$ shift term from \cref{rem:big-o-regret}:
$\textrm{coverage shock} = \left|\alpha_t^{*}-\alpha_{t-1}^{*}\right|.$
Together with the term that varies inversely with the window size $|W|$, this quantity dominates the regret. Minimizing coverage shock, i.e., the change in the optimal miscoverage level across consecutive timesteps, will improve regret performance. However, changes to the NCSF can have an outsized impact on this difference, motivating the counterfactual replay mechanism. Here, we present an intuitive explanation paired with evidence from a toy experiment in support.

The hindsight-optimal $\alpha_t^*$ depends on the distribution of non-conformity scores at time $t$. We decompose changes in this distribution into those arising from new observations and those induced by updates to the NCSF. First, we note that replacing a single data point yields a bounded change in the score distribution in \cref{prop:dset-change-bound}. Next, we show that the distance caused by the change in NCSF parameter is not bounded by the data-distribution shift.

\begin{remark}\label{rem:param-adaptation-cost}
    The score distribution depends on both the underlying data and the NCSF used to process it. We can decompose the distribution shift as follows due to \cref{prop:dset-change-bound}.
    \begin{equation}
        \TV(S_{\theta_t,t}, S_{\theta_{t+1},t+1}) \le \delta + \TV(S_{\theta_t,t+1}, S_{\theta_{t+1},t+1})
    \end{equation}
\end{remark}

\subsection{Computational Complexity Analysis}
\label{subsec: comp_complexity}
\vspace{-0.3em}
The computational complexity is dominated by replay associated costs, proportional to the size of the replay window $W$.
\begin{remark}[Computational Complexity of AdaptNC]
\label{thm:adaptcompcomplexity}
Let $W$ denote the window size, $k$ the number of experts, $M$ the number of Monte Carlo samples, $d$ the state dimension, $r$ the number of facets of the convex hull, and  $n = |\hat{R}_t|$ the number of high-density points retained after thresholding. Under the assumption $d \leq 3$, the amortized per-timestep cost of AdaptNC is $O\!\left(W + k + MWd/t_s \right).$
\end{remark}
We defer a detailed analysis of computational complexity to \cref{app:complexity}, and a discussion in \cref{app:low_dim_assump}

% \subsubsection{Empirical Evidence}

% \clearpage
% \section{Additional Experiments}
% We present results with $50$ different random seeds for each environment. We present the results showing the same metrics as the main paper with multiple seeds. Finally we present a sensitivity analysis with selections of the parameters $t_s$ and $w$, as requested.

% \begin{table}[ht]
% \centering
% \caption{Prediction Region Volume across 50 random seeds}
% \label{tab:volume_results}
% \begin{tabular}{lccc}
% \toprule
% Method & \textbf{Indoor Localization} & \textbf{Multirotor} & \textbf{Social Navigation} \\
% \midrule
% AdaptNC (replay)    & $27.7564 \pm 2.0770$  & $0.0365 \pm 0.0099$ & $43.5648 \pm 5.9982$ \\
% AdaptNC (no replay) & $38.6379 \pm 3.6671$  & $0.0709 \pm 0.0172$ & $73.4688 \pm 9.6169$ \\
% DtACI               & $43.0940 \pm 4.4557$  & $0.4361 \pm 0.0112$ & $52.6658 \pm 8.1048$ \\
% Split-CP            & $98.9110 \pm 3.1200$  & $0.2090 \pm 0.0001$ & $0.7061 \pm 0.0151$ \\
% Tumu et. al (2024)  & $102.0860 \pm 7.0275$ & $0.0027 \pm 0.0003$ & $0.6962 \pm 0.0189$ \\
% \bottomrule
% \end{tabular}
% \end{table}

    \vspace{-0.7em}
    \vspace{-0.3em}
\section{Experiments}\label{sec:experiments}
\vspace{-0.5em}

We evaluate AdaptNC across three case studies capturing common forms of distribution shift in robotics and dynamical systems: (i) environmental shift from unmodeled sensor noise, (ii) policy-induced shift from changing interaction dynamics, and (iii) dynamics shift from gradual system degradation. Together, these settings provide a comprehensive testbed for adaptive uncertainty quantification under evolving conditions.

We present three experimental settings. In the \textbf{Indoor Localization} experiments, a robot performs localization using wireless signal strength measurements. 
Distribution shifts arise in this setting because the underlying localization model fails to account for realistic wireless channel dynamics which vary over time.
% The distribution shift arises because the underlying location estimator does not account for temporal variations in the wireless channel caused by shadowing and small-scale fading effects. 
In the \textbf{Social Navigation} setting, $8$ agents navigate a shared environment based on information from agents within a collaboration radius. A trajectory predictor seeks to estimate the future poses of the agents. The distribution shift in this scenario stems from a gradual increase in the collaboration radius. In the \textbf{Multirotor Tracking} experiment, a predictor seeks to estimate the future position of a multirotor tracking a reference trajectory while maintaining a fixed altitude, but encounters distribution shift due to gradual degradation in the multirotor actuators. These three experiments showcase AdaptNC's benefits in settings due to environmental changes, policy shift in multi-agent settings, and actuator degradation. Further details are in Appendices~\ref{appendix:indoor-localization}, \ref{appendix:social-navigation}, and \ref{appendix:multirotor-tracking}.

For each case study, we evaluate all methods across $50$ different random seeds to assess performance variability and robustness. For each seed, evaluation is conducted over a single 6000-step rollout without resets, and results are aggregated across all seeds.

\textbf{Evaluation Metrics}: \textbf{(1) Global Coverage}, defined as $1 - \frac{1}{T}\sum_{t=0}^{T} \err_t$, which measures recovery of the target coverage level,  \textbf{(2)} Average Coverage Volume per Covered Timestep, which quantifies conservatism. \textbf{(3) Local Coverage} metric, defined by $1 - \frac{1}{w}\sum_{t-\frac{w}{2}+1}^{t+\frac{w}{2}} \err_t$, a sliding-window average of size $w$ that captures temporal behavior under distribution shift. 

\textbf{Baselines}: We compare against four baselines. \textbf{(1) DtACI:} As the state-of-the-art method for online conformal prediction under distribution shift, it serves as the primary and most relevant baseline as it consistently outperforms prior approaches \cite{gibbs_conformal_2024}. \textbf{(2) Split Conformal Prediction:} assumes exchangeability and does not account for distribution shift. \textbf{(3) AdaptNC d without Replay:} an ablation to assess the role of replay in stabilizing adaptation and mitigating coverage shock. \textbf{(4) Shape-CP:} \cite{tumu_multi-modal_2024} uses optimized but static nonconformity scores, highlighting the limitations of static calibration under distribution shift.

\begin{table*}[b]
\vspace{-1.8em}
\captionsetup{font=small}
\caption{This table reports the proposed evaluation metrics across all methods and case studies, averaged over 50 seeds. AdaptNC achieves coverage closest to the $90\%$ target while maintaining low uncertainty volumes. It also exhibits the most stable and robust behavior, as reflected by the low standard deviation of all metrics.}
\label{table:performance_metrics}
\centering
\resizebox{1.0\textwidth}{!}{
\begin{tabular}{lccccccccc}
\toprule
& \multicolumn{3}{c}{Indoor Localization} & \multicolumn{3}{c}{Social Navigation} & \multicolumn{3}{c}{Multirotor Navigation} \\
\cmidrule(lr){2-4}\cmidrule(lr){5-7}\cmidrule(lr){8-10}
Method              & Global Cov. & Vol. (\unit{\m\squared}) $\downarrow$  & Mean Local Cov. & Global Cov. &  Vol. (\unit{\m\squared}) $\downarrow$  & Mean Local Cov. & Global Cov. & Vol. (\unit{\m\squared}) $\downarrow$ & Mean Local Cov. \\
\midrule
\textbf{ANC}      
& \tbf{90.6\% $\pm$ 0.3\%} & \tbf{27.76 $\pm$ 2.08} & \tbf{90.2\% $\pm$ 0.3\%} 
& \tbf{91.6\% $\pm$ 1.2\%} & \tbf{43.56 $\pm$ 6.00} & \tbf{91.4\% $\pm$ 1.2\%} 
& \tbf{92.5\% $\pm$ 0.7\%} & \tbf{0.04 $\pm$ 0.01} & \tbf{92.3\% $\pm$ 0.7\%} \\

ANC NR  
& 94.8\% $\pm$ 0.8\% & 38.64 $\pm$ 3.67 & 94.4\% $\pm$ 0.8\%
& 87.9\% $\pm$ 3.0\% & 73.47 $\pm$ 9.62 & 87.6\% $\pm$ 3.0\%
& 80.1\% $\pm$ 2.5\% & 0.07 $\pm$ 0.02 & 80.0\% $\pm$ 2.5\% \\

DtACI               
& 92.2\% $\pm$ 1.8\% & 43.09 $\pm$ 4.46 & 91.8\% $\pm$ 1.8\%
& 80.0\% $\pm$ 3.5\% & 52.67 $\pm$ 8.10 & 79.7\% $\pm$ 3.5\%
& 87.4\% $\pm$ 1.7\% & 0.44 $\pm$ 0.01 & 87.2\% $\pm$ 1.7\% \\

Sp. CP            
& 97.1\% $\pm$ 0.9\% & 98.91 $\pm$ 3.12 & 96.7\% $\pm$ 0.9\%
& 3.0\% $\pm$ 0.8\% & \tbf{0.71 $\pm$ 0.02} & 3.0\% $\pm$ 0.8\%
& 60.1\% $\pm$ 1.6\% & 0.21 $\pm$ 0.00 & 59.9\% $\pm$ 1.6\% \\

Sh. CP            
& 96.9\% $\pm$ 1.1\% & 102.1 $\pm$ 7.0 & 96.5\% $\pm$ 1.1\%
& 3.0\% $\pm$ 0.8\% & 0.70 $\pm$ 0.02 & 3.0\% $\pm$ 0.8\%
& 0.7\% $\pm$ 0.2\% & \tbf{0.003 $\pm$ 0.0003} & 0.6\% $\pm$ 0.2\% \\

\bottomrule
\end{tabular}}
\vspace{-10pt}
\begin{minipage}{\textwidth}
\footnotesize
\textbf{Abbreviations:} ANC= AdaptNC; NR = No replay; Sh. CP = Shape CP; Sp. CP = Split CP; CP = Conformal Prediction.
\end{minipage}
\end{table*}

\vspace{-5pt}
\subsection{Performance Evaluation and Comparative Analysis}
\vspace{-5pt}

Table~\ref{table:performance_metrics} summarizes performance across all methods and environments. AdaptNC achieves coverage closest to the target level of $90\%$ while maintaining substantially smaller uncertainty volumes, indicating a favorable trade-off between coverage and efficiency. Furthermore, the low variance across 50 seeds shows that these trends remain consistent, demonstrating the robustness of the method.

% \textbf{Global Coverage:}
% AdaptNC consistently achieves coverage nearest to the target across all settings (e.g., $90.6\%$, $91.6\%$, and $92.5\%$). In contrast, baseline methods exhibit inconsistent behavior. DtACI attain near-target coverage in certain settings, but this is typically accompanied by significantly larger uncertainty regions. Conversely, methods that maintain small volumes, such as Split CP and Shape-CP in some environments, suffer from severe undercoverage (e.g., $\approx 3.0\%$ in social navigation). These results indicate that the baselines cannot maintain coverage without incurring substantial inefficiency.

% \textbf{Local Coverage Stability:}
% AdaptNC exhibits the most stable local coverage across all environments, with consistently low variance. In contrast, baseline methods display significant temporal variability. AdaptNC without replay, for instance, attains reasonable global coverage only through oscillatory local behavior, which necessitates larger uncertainty regions. Similarly, DtACI and Split CP exhibit intermittent coverage drops under distribution shift. This instability further reinforces the trade-off, as maintaining stable coverage requires either adaptivity or increased conservatism.
\begin{figure*}[t]
    \centering
    \includegraphics[width=0.8\linewidth]{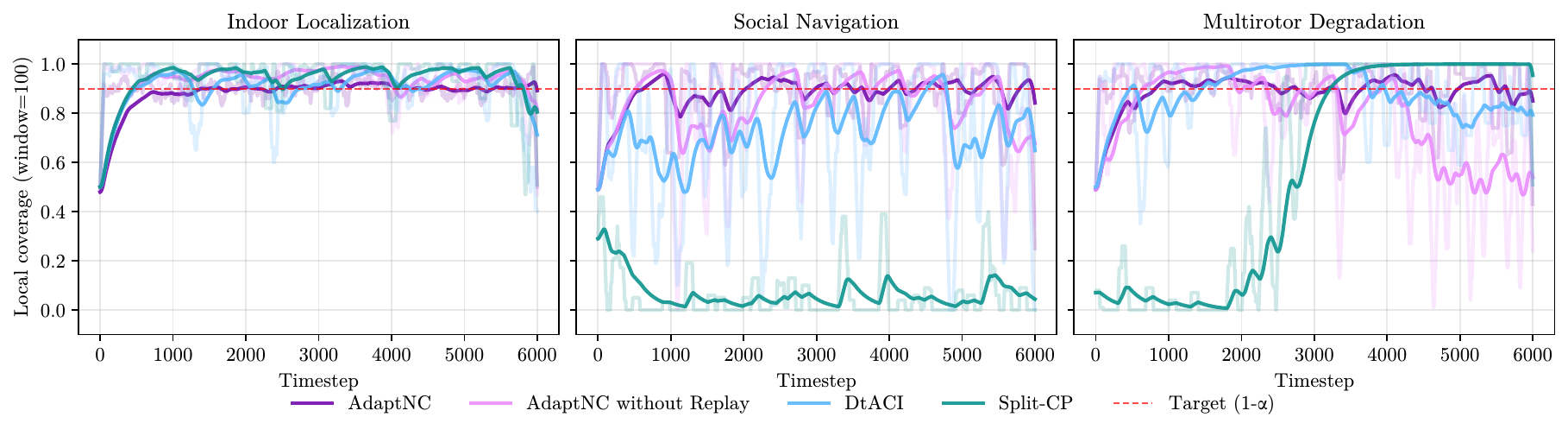}
    \vspace{-0.8em}
\captionsetup{font=small}
\caption{This figure shows the \textbf{evolution of local coverage over a 100-step sliding window} for a representative seed. For readability, an exponential moving average is shown, with the raw data in a lighter color. AdaptNC exhibits the lowest variability in local coverage, indicating stable recovery of tight uncertainty regions under distribution shift.} 
\label{fig:local_baseline_comparison}
\vspace{-1.2em}
\end{figure*}
\begin{figure*}[t]
    \centering
    \includegraphics[width=0.8\linewidth]{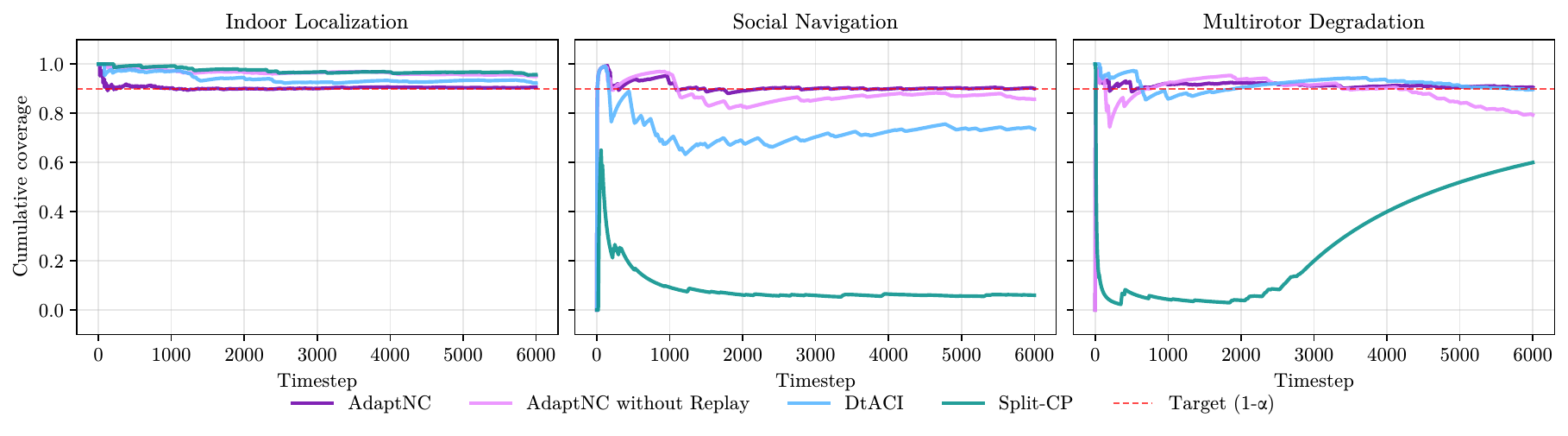}
    \vspace{-0.8em}
\captionsetup{font=small}
\caption{This figure illustrates the \textbf{evolution of empirical coverage over the evaluation horizon} relative to the target coverage level (shown in red) \textit{for a representative seed}. AdaptNC's performance highlights that score-function adaptation enables sustained coverage with tighter regions.} 
\label{fig:global_baseline_comparison}
\vspace{-1.5em}
\end{figure*}

\textbf{Global Coverage and Local Coverage Stability:}
AdaptNC consistently achieves coverage closest to the target across all settings ($90.6\%$, $91.6\%$, and $92.5\%$) while maintaining stable local coverage with low variance. In contrast, baseline methods exhibit a clear trade-off between coverage, stability, and efficiency. DtACI attains near-target coverage in some settings, but only with significantly larger uncertainty regions. Conversely, methods that achieve small volumes, such as Split CP and Shape-CP in certain environments, suffer from severe undercoverage (e.g., $\approx 3.0\%$ in social navigation). Moreover, these baselines display pronounced temporal variations (Figures~\ref{fig:local_baseline_comparison} and~\ref{fig:global_baseline_comparison}): AdaptNC without replay achieves reasonable global coverage only through oscillatory local behavior, while DtACI and Split CP exhibit intermittent coverage drops under distribution shift. These results indicate that baseline methods cannot simultaneously maintain accurate, stable coverage and efficient uncertainty regions, whereas AdaptNC achieves all three.

\textbf{Uncertainty Volume:}
AdaptNC achieves smaller uncertainty regions while maintaining near-target coverage. For example, in multirotor tracking, AdaptNC attains comparable coverage to DtACI with an $11$x reduction in volume ($0.04$ vs.\ $0.44$). A similar trend is observed in indoor localization, where DtACI achieves reasonable coverage but at significantly higher volume ($43.09$ vs.\ $27.76$). In contrast, methods that achieve low volumes do so by sacrificing coverage, as seen with Split CP. Thus, baseline methods either achieve efficiency at the cost of coverage or achieve coverage with excessive volume. Figure~\ref{fig: shape_evolution} further illustrates that shape adaptation enables significantly tighter conformal prediction regions while maintaining coverage.

\textbf{Importance of Replay:} Figures~\ref{fig:local_baseline_comparison} and~\ref{fig:global_baseline_comparison} illustrate the evolution of global and local coverage over time for all case studies using a representative seed, highlighting the role of the replay mechanism. A consistent trend is observed across settings: AdaptNC maintains coverage closer to the target with significantly more stable behavior. In contrast, AdaptNC without replay exhibits increased variability in local coverage, which propagates to oscillations in global coverage and leads to more conservative uncertainty regions. These results highlight the importance of replay in stabilizing coverage under distribution shift and enabling the recovery of tighter, more accurate uncertainty regions.

\textbf{Runtime Analysis:} We report per-step runtimes in Table~\ref{table:runtime}, averaged over 50 seeds on a desktop workstation (AMD Ryzen 7900X, 64GB RAM). Runtimes include all operations required to return a prediction region, including score evaluation (for offline methods Split-CP and Shape-CP). The threshold update in AdaptNC has a runtime comparable to DtACI, under $0.1$ms in all environments. The score update averages $\approx 100$ ms, dominated by KDE evaluation, consistent with the theoretical analysis. Score updates are performed periodically (every $20$, $100$, and $100$ steps across the three case studies), keeping latency within the $t_s \cdot \Delta t$ budget. Importantly, periodic updates do not degrade coverage, as threshold updates occur at every timestep, which we verify empirically. Moreover, a runtime of $\approx 100$ ms is comparable to static safety filtering methods such as Control Barrier Functions (CBFs)~\cite{Ames_2017}, which must be recomputed in settings with distribution shift. Further discussion of these topics are provided in Appendices~\ref{app:periodic_safety} and~\ref{app:CBF_comp} respectively. 
% Finally, the $\approx 90$ ms gap between AdaptNC and the no-replay variant isolates the cost of replay, indicating that it constitutes majority of the score update overhead.
\begin{table}[b]
\vspace{-2.0em}
\captionsetup{font=small}
\caption{Per-step runtime (ms) of all methods, averaged over 50 seeds.}
\label{table:runtime}
\centering
% \small
% \setlength{\tabcolsep}{4pt}
% \renewcommand{\arraystretch}{0.75}
\resizebox{0.73\textwidth}{!}{
\begin{tabular}{lccc}
\toprule
Method & Indoor Localization (ms) & Multirotor (ms) & Social Navigation (ms) \\
\midrule
AdaptNC (threshold update) & 0.09 $\pm$ 0.00 & 0.08 $\pm$ 0.00 & 0.08 $\pm$ 0.00 \\
AdaptNC (score update) & 100.61 $\pm$ 48.28 & 101.90 $\pm$ 38.06 & 97.34 $\pm$ 38.09 \\
AdaptNC (no replay, threshold update) & 0.15 $\pm$ 0.00 & 0.16 $\pm$ 0.01 & 0.13 $\pm$ 0.00 \\
AdaptNC (no replay, score update) & 7.04 $\pm$ 0.47 & 6.91 $\pm$ 0.12 & 6.55 $\pm$ 0.18 \\
DtACI & 0.09 $\pm$ 0.00 & 0.08 $\pm$ 0.00 & 0.08 $\pm$ 0.00 \\
Split CP & 0.00 $\pm$ 0.00 & 0.00 $\pm$ 0.00 & 0.00 $\pm$ 0.00 \\
Shape CP & 0.22 $\pm$ 0.06 & 0.20 $\pm$ 0.40 & 0.21 $\pm$ 0.04 \\
\bottomrule
\end{tabular}}
\vspace{-20pt}
\end{table}

\textbf{Hyperparameter Sensitivity:} We also provide a sensitivity analysis across different values of shape adaptation interval (update period) and replay buffer window size to assess the robustness of AdaptNC to moderate hyperparameter changes. Due to space constraints, the results are reported in Table~\ref{table:sensitivity} in Appendix~\ref{app:sensitivity_analysis}. Overall, AdaptNC exhibits low sensitivity to hyperparameters, with stable performance across a broad range of window sizes and update periods. Coverage and volume change noticeably only when deviating far from the base setting. The results show changes in the window size and update period have monotonic effects on both coverage and volume. The predictable nature of these tradeoffs enables straightforward, principled tuning.

\begin{figure*}[t]
    \centering
    \includegraphics[width=0.8\linewidth]{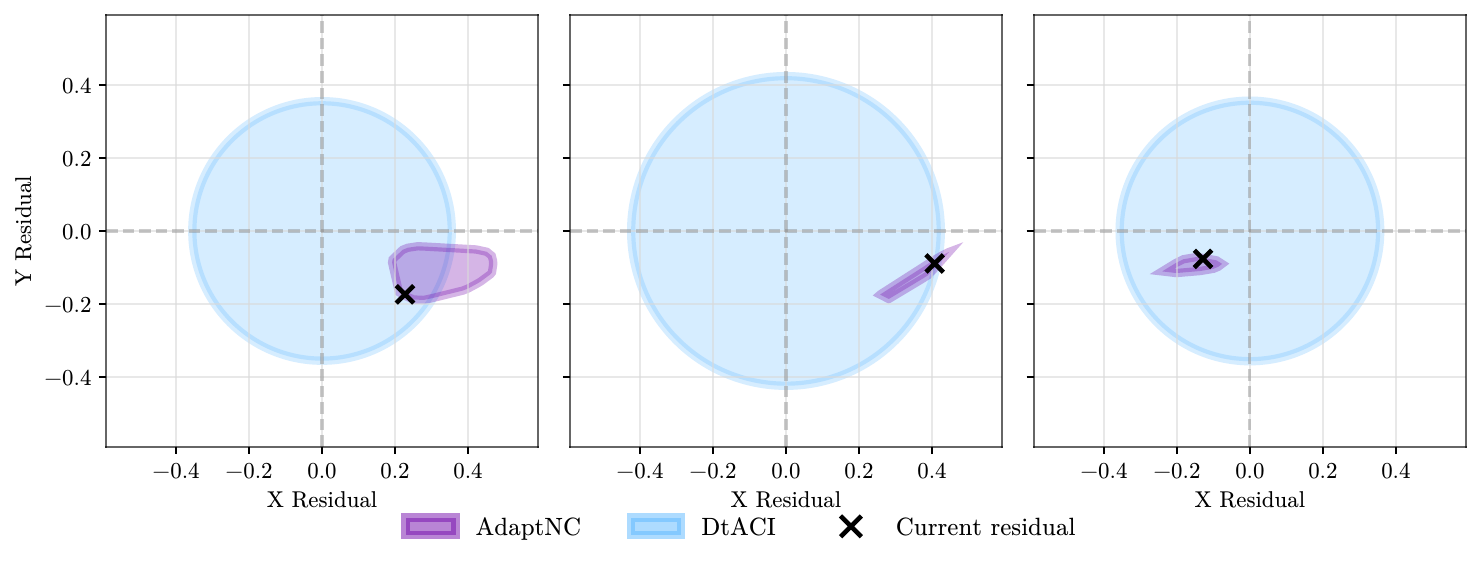}
    \vspace{-0.8em}
\captionsetup{font=small}
\caption{This figure shows uncertainty regions from AdaptNC and DtACI at representative timesteps ($t\in[720,1560,3120]$) in the multirotor tracking task for a representative seed, along with the realized residual. \textbf{AdaptNC recovers significantly tighter regions} while maintaining coverage, whereas DtACI’s threshold-based adaptation constrains region geometry and leads to misalignment under distribution shift.} 
\label{fig: shape_evolution}
\vspace{-2.0em}
\end{figure*}
\textbf{Results Summary:} Across all three case studies, AdaptNC maintains near-target coverage under structural distribution shifts while producing significantly tighter uncertainty regions than competing methods. By jointly adapting the nonconformity score and threshold, it avoids the conservatism of threshold-only approaches and the undercoverage of non-adaptive baselines, with replay mitigating coverage shocks and stabilizing local coverage. This behavior is also reflected in the expert weight dynamics (Figures~\ref{fig:appendix-indoor-localization-weights},~\ref{fig:appendix-socialnav-weights}, and \ref{fig:appendix-multirotor-weights}), where AdaptNC exhibits faster reweighting than DtACI and AdaptNC without replay. Additional details are provided in Appendices~\ref{appendix:indoor-localization},~\ref{appendix:social-navigation}, and~\ref{appendix:multirotor-tracking}.

    \vspace{-0.7em}
    \section{Conclusion}
\vspace{-0.5em}
We present AdaptNC, an online conformal prediction method that adapts nonconformity scores under distribution shift. It leverages a reweighted data distribution and a replay mechanism to stabilize adaptation and maintain coverage as scores evolve. We establish regret and long-term coverage guarantees and validate the method on three robotics case studies involving \textbf{environmental change, multi-agent interaction,} and \textbf{system degradation}. Across all settings, AdaptNC achieves coverage closer to the target with lower uncertainty volume than competing methods. Future work will focus on improving computational efficiency and addressing the limitations discussed in Appendix~\ref{sec:limitations}.

    \bibliography{references}
    \bibliographystyle{unsrtnat}

    \newpage
    \appendix
    \etocdepthtag.toc{appendix}
    \onecolumn
    \begingroup
      \etocsettagdepth{main}{none}      % suppress everything before appendix 
      \etocsettagdepth{appendix}{all}  % show appendix entries
      \setcounter{tocdepth}{2}
    
      % Optional: custom title styling for this TOC
      \etocsettocstyle{\section*{Appendix Contents}}{} % title + after-title hook
    
      \tableofcontents
    \endgroup
    % \section*{Table of Contents for Appendix}
    % \begin{enumerate}[label=\Alph*., leftmargin=2em]
    %     \item \hyperref[sec:limitations]{Limitations} \dotfill \pageref{sec:limitations}
    %     \item \hyperref[appendix:method]{Extended Methodology and Implementation} \dotfill \pageref{appendix:method}
    %     \begin{enumerate}[label=\theenumi\arabic*.]
    %         \item \hyperref[sec:score-optim]{Score Optimization Algorithm} \dotfill \pageref{sec:score-optim}
    %         \item \hyperref[sec:adaptnc-algorithm]{AdaptNC Algorithm} \dotfill \pageref{sec:adaptnc-algorithm}
    %         \item \hyperref[sec:dtaci]{Dynamically Tuned Adaptive Conformal Inference} \dotfill \pageref{sec:dtaci}
    %         \item \hyperref[sec:mc-kde]{High Density Region Identification}\dotfill\pageref{sec:mc-kde}
    %     \end{enumerate}        
    %     \item \hyperref[appendix:theory]{Theoretical Results} \dotfill \pageref{appendix:theory}        
    %     \item \hyperref[appendix:experiments]{Extended Experiments} \dotfill \pageref{appendix:experiments}
    %     \begin{enumerate}[label=\theenumi\arabic*.]
    %         \item \hyperref[appendix:indoor-localization]{Indoor Localization} \dotfill \pageref{appendix:indoor-localization}
    %         \item \hyperref[appendix:social-navigation]{Social Navigation} \dotfill \pageref{appendix:social-navigation}
    %         \item \hyperref[appendix:multirotor-tracking]{Multirotor Tracking} \dotfill \pageref{appendix:multirotor-tracking}
    %     \end{enumerate}
    %     \item \hyperref[appendix:experiments]{Additional Details on Multirotor Experiments} \dotfill \pageref{appendix:experiments}
    % \end{enumerate}
    % \NT{MAKE SURE TO UPDATE TOC}
    
    \newpage
    \section{Limitations}\label{sec:limitations}

While AdaptNC demonstrates the ability to recover tight, valid uncertainty regions under distribution shift, it remains, as with other adaptive conformal prediction methods, susceptible to producing trivial uncertainty sets. Although such sets are not entirely uninformative, as they indicate that future observations are likely to fall outside the previously observed support and that the system may be undergoing a regime change, they nonetheless represent a limitation. Empirically, AdaptNC exhibits a vacuous coverage rate of approximately $25\%$ across all case studies, while DtACI exhibits a rate of approximately $20\%$. Although AdaptNC shows a slightly higher vacuous coverage rate, the difference is marginal, and these results suggest that vacuous coverage is an inherent limitation shared across all adaptive conformal prediction methods considered in this study. Furthermore, we emphasize that vacuous coverage steps do not constitute a critical safety concern in autonomous systems. Rather, they serve as indicators of a potential regime change in the environment, at which point the system can transition to a fallback safe controller during the transient phase. We believe that score function regularization may be able to address this issue, and it is an avenue of future work we are pursuing.

% To further examine this behavior, we report the \textbf{percentage of vacuous coverage timesteps} for each method in
% Table~\ref{table:vacuous_coverage}, defined as $\frac{1}{T}\sum_{t=0}^{T} \mathds{1}\{\hat{q}_{t,1-\alpha} = \infty\}$, which quantifies the frequency with which a method yields uninformative, unbounded uncertainty regions.

% \begin{table}[h]
% \centering
% % \resizebox{\columnwidth}{!}{%
% \begin{tabular}{lccc}
% \toprule
% Method 
% & Indoor Localization 
% & Social Navigation 
% & Multirotor Navigation \\
% \midrule
% AdaptNC              & 11.51\% & 40.42\% & 33.65\% \\
% AdaptNC w/o Replay   & 11.93\% & 24.52\% & 22.82\% \\
% DtACI                & 15.67\% & 21.17\% & 28.77\% \\
% Split CP             & 0.00\%  & 0.00\% & 0.00\%  \\
% Shape CP             & 0.00\%  & 0.00\% & 0.00\%  \\
% \bottomrule
% \end{tabular}
% % }
% \caption{This table lists the \textbf{percentage of timesteps exhibiting vacuous coverage} for each method across all case studies. \textbf{Methods incorporating temporal adaptation exhibit a nonzero frequency of vacuous coverage.}}
% \label{table:vacuous_coverage}
% \vspace{-1.0em}
% \end{table}

% Table~\ref{table:vacuous_coverage} shows that Split CP does not exhibit vacuous coverage at any time step, as it lacks the ability to adapt over time. 

\section{Additional Clarifications and Extended Discussion}

\subsection{AdaptNC is not a direct combination of DtACI and prior score-adaptation methods}~\label{app:direct_comb}
While AdaptNC builds upon these important prior works, it is not merely a direct combination of these methods. The key novelty of AdaptNC lies in identifying and addressing a new setting that prior work does not consider: the joint, coupled adaptation of both the nonconformity score and the conformal threshold under distribution shift. This coupling introduces a non-trivial feedback loop. Changes in the score function alter the induced nonconformity distribution, which subsequently shifts the optimal quantile and destabilizes threshold updates. As a result, the update dynamics of prior independent methods no longer perform their function in the same manner. To support this new setting, we introduced an adaptive reweighting scheme that enables fully online score optimization, alongside new theoretical results including score function stability under joint adaptation (Theorem 5.3) in an asymptotic setting, using a guarantee on the stability of the score function adaptation in (Theorem B.1).

Furthermore, the direct amalgamation of DtACI with score adaptation, corresponds to our “AdaptNC without Replay” baseline, which includes our historical data reweighting scheme. As demonstrated in our ablation studies, this naive combination fails, exhibiting substantial misses in coverage and significantly inflated prediction regions because composition alone is insufficient to handle the feedback problem created through this joint optimization. The intuition behind this is presented in (Remark 5.5), which shows that there are two sources of shift in the score distribution; the bounded component due to the additional data point observed at the next timestep, and the potentially unbounded component which is due to the shift in the score function. To solve this, we introduced the counterfactual replay mechanism. By reinitializing and recalibrating the DtACI expert chain using recent data under the updated score, the replay mechanism prevents coverage shocks caused by mismatched quantiles. As shown empirically, replay is not a minor implementation detail but a fundamentally necessary component that restores both valid coverage and tight uncertainty sets.

Overall, AdaptNC goes significantly beyond combining existing techniques by formulating a new problem setting of coupled adaptation, identifying the instabilities of a naive synthesis, and introducing the novel mechanisms required to make joint adaptation both stable and effective.

\subsection{AdaptNC's Theoretical Support}
AdaptNC's various components have asymptotic guarantees. It is important to note that finite-sample guarantees are generally unavailable in online conformal prediction (CP) due to the absence of assumptions on distribution shift. Guarantees are often given in the asymptotic setting. \cite{gibbs_conformal_2024} show that the regret of their expert weighting algorithm will have a bounded regret to the best expert. \cite{gibbs_adaptive_2021} establish asymptotic convergence of each expert to the target coverage. In this asymptotic stability setting, we guarantee that our method yields tight and consistent prediction regions. This is shown in \cref{thm:score_func_stability,cor:prediction_region_size_stability}. Taken together, our method provides convergence guarantees in the asymptotic setting, including for the score adaptation procedure. Our theoretical analysis gives sufficient conditions under which replay induces a vanishing average perturbation of the ACI expert recursion. These conditions are asymptotic and rely on boundary anti-concentration; finite-sample replay guarantees remain an important direction for future work.

\subsection{Role of the Low-Dimensional Assumption in Complexity Analysis}~\label{app:low_dim_assump}
The assumption \( d \leq 3 \) in Section~\ref{subsec: comp_complexity} is introduced to derive tractable bounds for the geometric operations underlying AdaptNC, particularly convex hull construction via the QuickHull algorithm. In low dimensions, QuickHull exhibits near-linear complexity, which supports the stated amortized per-timestep bound. This assumption is therefore analytical in nature and is used to characterize the regime where AdaptNC achieves favorable computational efficiency.

Importantly, AdaptNC itself does not depend on \( d \leq 3 \) and remains applicable in higher dimensions, although the associated geometric computations scale less favorably. In particular, QuickHull has complexity \( O(n^{\lfloor d/2 \rfloor}) \) for $d > 3$, where \( \lfloor \cdot \rfloor \) denotes the greatest integer function. However, this growth remains manageable up to moderate dimensions (e.g., $d \leq 5$), where the scaling is at most quadratic, consistent with the empirical runtimes in Table~\ref{table:runtime}. Moreover, in autonomous systems applications, which are the primary focus of this work, conformal prediction is typically performed over low-dimensional predictors such as \textbf{scalar} value functions or spatial trajectory predictors (up to $3$ dimensions), and thus AdaptNC remains applicable without incurring prohibitive computational overhead.

\subsection{Effects of Vacuous Prediction Regions on the Next Cycle of Learning and Their Mitigation}
The "infinite" safety margins or vacuous bounds do not corrupt the next cycle of learning, as the underlying data is not perturbed or destroyed. These vacuous bounds are often encountered when the system undergoes change, or sees significant shift in error patterns. When the non-conformity scores the system sees are very large, larger than those we have previously seen, AdaptNC correctly identifies that the system is undergoing change, and surfaces to the user that the change is major, and the system cannot yet accurately predict what will happen. In this way, vacuous bounds \textit{are} informative; they show that AdaptNC cannot provide guarantees in high-change regimes. After this period of change is complete, AdaptNC will return to predicting safety regions, as the residuals stabilize. Finally, prediction of these vacuous bounds happens when the current quantile prediction $\bar{\alpha}_t$ becomes very small. The data is preserved in its unperturbed state in the history buffer $\mathcal{H}$. These independent mechanisms ensure that data is kept in a manner that permits learning from samples, even if the predicted safety margin is infinite.

\subsection{Detecting Distribution Shift and Adapting to its Rate}
Identifying when the data distribution has shifted is one of the major problems in the online conformal prediction setting. There are advantages to having more data, namely that more data allows for more data to develop bounds on the non-conformity scores. However, using old or stale data can yield inaccurate coverage which can lead to catastrophe for robots. AdaptNC is fundamentally a weighted mixture of a series of ``expert" predictors. Each of these ``experts" functions under the hypothesis that the distribution is shifting at a different rate, characterized by $\gamma$. A lower value of $\gamma$ indicates a low expected rate of change, and a higher $\gamma$ indicates a higher expected rate of change. AdaptNC uses the weights allocated to each expect to dynamically identify how quickly the distribution is changing, and use that estimate to change how much historical data it uses at every timestep.  This is also illustrated in the expert weight dynamics (Figures~\ref{fig:appendix-indoor-localization-weights},~\ref{fig:appendix-socialnav-weights}, and \ref{fig:appendix-multirotor-weights}).

\subsection{Ability to Maintain Safety if a Catastrophic Change Occurs Between Score Updates}~\label{app:periodic_safety}
AdaptNC separates adaptation into two components: The threshold update occurs at every timestep, while the score update occurs periodically. While the shapes update periodically, the threshold updates at every timestep. This provides instantaneous response to changes in the system, while the calculation of the
new score function allows for AdaptNC to react to systemic changes in the system. We show empirically in our results that in scenarios with distribution shift, our method is able to maintain the target coverage better than the previous state of the art.

\subsection{Tackling More Complex Tasks and Computational Considerations}~\label{app:CBF_comp}
While the task of drawing safety boundaries may appear computationally intensive, in practice it is not prohibitively slow. In our experimental settings, the score update path executed every $T$ steps takes approximately $100$ ms, while the threshold update path executed at every timestep takes approximately $0.1$ ms. The threshold update runtime is comparable to standard safety filtering methods such as Control Barrier Functions (CBFs), which typically assume negligible or no distribution shift. The score update enables adaptation to significant distribution shift. In such settings, recomputing CBFs is often more computationally demanding than the threshold update in AdaptNC.

\clearpage
    \section{Method}\label{appendix:method}

\subsection{Score Optimization Algorithm}\label{sec:score-optim}
\begin{algorithm}[]
    \caption{OptimizeScore: Score Parameter Optimization}
    \label{alg:optimizescore}
    \begin{algorithmic}[1]
        \REQUIRE History buffer $\mathcal{H}_t^w$, current parameters $\theta$, target miscoverage rate $\alpha$
        % \STATE Compute weights $\omega_t$ for each sample $(X_t, Y_t) \in \mathcal{H}$ using \cref{eq:weight_calc}
        \STATE Generate high-density points $\hat{R}_t$ covering proportion $1-\alpha$ of density mass through \cref{alg:mckde}
        \STATE Fit shape template $\mathcal{S}$ using QuickHull algorithm on $\hat{R}_t$
        \STATE $\theta_{t+1} \gets$ parameters encoding shape templates $\{\mathcal{S}\}$
        \STATE \textbf{Return:} Updated parameters $\theta_{t+1}$
    \end{algorithmic}
\end{algorithm}

\subsection{AdaptNC Algorithm}\label{sec:adaptnc-algorithm}
\begin{algorithm}[]
    % \small
    \caption{AdaptNC Algorithm}
    \label{alg:adaptnc}
    \begin{algorithmic}
        \REQUIRE Target miscoverage rate $\alpha$, adaptation interval $t_s$, initial parameters $\theta_0$, DtACI expert gammas $\Gamma = [\gamma_1, \ldots, \gamma_k]$, window size $W$, learning rate $\eta$, calibration data $\mathcal{D}_{cal}$
        \STATE Initialize history buffer $\mathcal{H} \gets \emptyset$
        \STATE Fit initial conformal region with $\theta_0$ on $\mathcal{D}_{cal}$
        \STATE Initialize DtACI with $s(\cdot; \theta_0)$, $\alpha$, $\Gamma$, $W$, $\eta$, $\mathcal{D}_{cal}$
        \FOR{$t = 1, 2, \ldots, T$}
            % \STATE \textbf{// Observation Step}
            \STATE Observe $X_t$, produce prediction $\hat{Y}_t$
            \STATE Output prediction region $\hat{C}_t(X_t;\theta_t,q_{1-\alpha_t,\mathcal{D}_t})$
            \STATE Observe $Y_t$
            \STATE Append $(X_t, Y_t, s_t)$ to $\mathcal{H}$
            % \STATE \textbf{// Conformal Threshold Update}
            \STATE Compute score $s_t = s(X_t, Y_t; \theta_t)$
            \STATE Update DtACI with $(X_t, Y_t)$ to obtain $\bar{\alpha}_t$
            \STATE Obtain new threshold: $q_{1-\alpha_t,\mathcal{D}}$
            \IF{$t \bmod t_s = 0$}
                % \STATE \textbf{// Score Parameter Adaptation}
                \STATE Compute distribution $\mathcal{H}_t^\omega$ using \cref{eq:weight_calc}
                \STATE $\theta_{t+1} \gets \textsc{OptimizeScore}(\mathcal{H}_t^\omega, \theta_t, \alpha)$
                % \STATE \textbf{// Replay DtACI with Updated Score}
                % \STATE $\mathcal{D}_{replay} \gets$ last $W$ observations from $\mathcal{H}$
                \STATE Re-initialize DtACI with score function $s(\cdot; \theta_{t+1})$
                \FOR{each $(X_i, Y_i) \in [t-w+1,t]$}
                    \STATE Update experts and expert weights with $(X_i, Y_i)$
                \ENDFOR
                % \STATE Adjust conformal region shapes with updated threshold
            \ENDIF
        \ENDFOR
    \end{algorithmic}
\end{algorithm}

\subsection{Dynamically Tuned Adaptive Conformal Inference (DtACI)}\label{sec:dtaci}
See \cref{alg:dtaci}.
\begin{algorithm}[h]
   % The optional argument [DtACI] prevents the citation from breaking the List of Algorithms
   % Using the safer caption format we discussed
   \caption{Dynamically tuned Adaptive Conformal Inference (DtACI) \citep[Algorithm 2]{gibbs_conformal_2024}}
   \label{alg:dtaci}
%    \NT{This should go to appendix}
   \begin{algorithmic}[1]
        \REQUIRE Observed values $\{\beta_t\}_{1\leq t \leq T}$, set of candidate $\gamma$ values $\{\gamma_i\}_{1\leq i \leq k}$, starting points $\{\alpha_1^i\}_{1\leq i \leq k}$, and parameters $\sigma$ and $\eta$.
        \STATE $w_1^i \gets 1,\; 1\leq i\leq k;$
        \FOR{$t=1,2,\ldots,T$}
            \STATE Define the probabilities $p_t^i := w_t^i / \sum_{1\le j \le k} w_t^j,$
            $\quad \forall 1 \le i \le k$;
            \STATE Output $\bar{\alpha}_t = \sum_{1\le i \le k} p_t^i \alpha_t^i$;
            \STATE $\bar{w}_t^i \gets w_t^i \exp(-\eta \ell(\beta_t, \alpha_t^i)), \quad \forall 1 \le i \le k$;
            \STATE $\bar{W}_t \gets \sum_{1\le i \le k} \bar{w}_t^i$;
            \STATE $w_{t+1}^i \gets (1-\sigma)\bar{w}_t^i + \bar{W}_t \sigma / k$;
            \STATE $\err_t^i := \mathds{1}\{Y_t \notin \hat{C}_t(\alpha_t^i;\theta_t,\mathcal{D}_t)\}, \quad \forall 1 \le i \le k$;
            \STATE $\err_t := \mathds{1}\{Y_t \notin \hat{C}_t(\bar{\alpha}_t;\theta_t,\mathcal{D}_t)\}$;
            \STATE $\alpha_{t+1}^i = \alpha_t^i + \gamma_i (\alpha - \err_t^i), \quad \forall 1 \le i \le k$;
        \ENDFOR
   \end{algorithmic}
\end{algorithm}

\subsection{High Density Region Identification}\label{sec:mc-kde}
\begin{algorithm}[h]
   \caption{Monte Carlo Kernel Density Estimation (MCKDE)}
   \label{alg:mckde}
\begin{algorithmic}[1]
   \STATE {\bfseries Input:} Dataset $X \in \mathbb{R}^{N \times d}$, sample weights $w \in \mathbb{R}^N$, Monte Carlo samples $M$, miscoverage rate $\alpha$, bandwidth factor $\beta$.
   \STATE {\bfseries Output:} Set of high-density samples $\hat{R}_{N,M}$.

   \STATE \textbf{// 1. Bandwidth Selection}
   \IF{bandwidth method is ``Scott''}
       \STATE $h \leftarrow N^{\frac{-1}{d+4}}$
   \ELSIF{bandwidth method is ``Silverman''}
       \STATE $h \leftarrow \left( \frac{N(d+2)}{4} \right)^{\frac{-1}{d+4}}$
   \ENDIF
   \STATE $h \leftarrow h \times \beta$

   \STATE \textbf{// 2. Fit Kernel Density Estimator}
   \STATE Fit weighted Gaussian KDE $\hat{f}_N$ using $(X, w)$

   \STATE \textbf{// 3. Monte Carlo Sampling \& Scoring}
   \STATE Draw $M$ samples $\mathcal{Z} = \{z_1, \dots, z_M\}$ where $z_i \sim \hat{f}_N$
   \STATE Calculate density scores $\xi_i \leftarrow \hat{f}_N(z_i)$ for $i=1, \dots, M$

   \STATE \textbf{// 4. Thresholding}
   \STATE Let $\tau$ be the $\alpha$-quantile of the set $\{\xi_1, \dots, \xi_M\}$
   \STATE $\hat{R}_{N,M} \leftarrow \{ z_i \in \mathcal{Z} \mid \xi_i \geq \tau \}$

   \STATE \textbf{return} $\hat{R}_{N,M}$
\end{algorithmic}
\end{algorithm}

Under standard KDE regularity conditions and sufficiently dense Monte Carlo
sampling, the retained high-density samples approximate the corresponding KDE
level set. We therefore model the score-optimization step through
Assumption~\ref{assump:template_consistency}.

\begin{assumption}[Consistent residual-template estimation]
\label{assump:template_consistency}
At score-update times \(t_m\), let \(\widehat R_m\) be the residual set
returned by the score-optimization procedure and let
\(K_m=\operatorname{conv}(\widehat R_m)\). Let
\(R_\star=\{z:f_\star(z)\ge \tau_\star\}\) and
\(K_\star=\operatorname{conv}(R_\star)\). We assume
\[
d_H(K_m,K_\star)\xrightarrow{p}0.
\]
\end{assumption}

\begin{theorem}[Glivenko-Cantelli Theorem as written in \citet{vaart_asymptotic_1998}]
\label{thm:gliv-cant-thm}
If $X_1, X_2, \dots, X_n$ are independent and identically distributed random variables with cumulative distribution function $f$, and empirical distribution function $f_n$, then:
\begin{equation}
    \|f_n - f\|_\infty \xrightarrow{as} 0
\end{equation}
\end{theorem}

% \begin{corollary}[Corollary 2.1 from \citet{cadre_kernel_2006}]
% \label{thm:cadre} This corollary is a consequence of \citet[Theorem 2.1]{cadre_kernel_2006}
% Let $d\geq 2$ be the dimension of the data, and $(a_N)_N$ be a sequence of positive real numbers such that $a_N \to 0$. Let $t^{(p)}$ be the unique real number such that $\lambda_f(L(f, t^{(p)})) = p$, and $t_N^{(p)}$ be the unique real number such that $\lambda_{\hat{f}_N}(L(\hat{f}_N, t_N^{(p)})) = p$. If the assumptions of \cref{thm:convergence} hold, then for almost every $t\geq 0$ we have that:
% \begin{equation}
%     \lambda(\hat{R}_{M,N,p} \Delta R_p) \xrightarrow{P} \frac{\sqrt{t_N^{(p)}}}{\sqrt{Nh^d} \varphi_N}\sqrt{\frac{2t}{\pi} \int K^2d\lambda}
% \end{equation}
% Where $\varphi_N = a_N/\lambda(\hat{R}_{M,N,p} - \hat{R}_{M,N,p+a_N})$ 
% \end{corollary}
% Note that in this corollary, the square root of the kernel term is a constant that depends on the choice of kernel, and the term $\sqrt{Nh^d}$ grows to infinity as $N$ increases, ensuring that the estimation error decreases with larger sample sizes. The term $t_N^{(p)}$ represents the density threshold corresponding to the level set, which is bounded for a well-defined density function. Thus, the corollary indicates that the estimation error diminishes as the sample size increases. The asymptotic behaviour described in \cref{thm:cadre} is that

% \begin{equation}
%     \lambda(\hat{R}_{M,N,p} \Delta R_p) \xrightarrow{P} O\left(1/\sqrt{Nh^d}\right)
% \end{equation}

\subsection{Computational Complexity Analysis of AdaptNC}~\label{app:complexity}

\begin{theorem}[Computational Complexity of AdaptNC]

AdaptNC consists of two components: a threshold update executed at every timestep and a score function update executed every $t_s$ timesteps. Let $W$ denote the window size, $k$ the number of experts, $M$ the number of Monte Carlo samples, $d$ the residual dimension, $r$ the number of facets of the convex hull, and $n = |\hat{R}_t|$ the number of retained high-density points. Then, the per-timestep cost of the threshold update is $O(r + W + k)$, and the cost of the score function update is $O\!\left(MWd + M \log M + n \log n + W(r+k)\right)$, which is dominated by the KDE evaluation term $O(MWd)$ in the low-dimensional regime considered here.

Under the assumption $d \leq 3$, the amortized per-timestep cost of AdaptNC is
\[
O\!\left(W + k + MWd/t_s\right).
\]
\end{theorem}

\begin{proof}
We analyze the computational cost of AdaptNC by decomposing it into two components: the threshold update and the score function update. We assume that the residual dimensionality satisfies $d \leq 3$, which holds for all experiments and typical real-world trajectory prediction tasks.

\textbf{Threshold update:}
The threshold update executes at every timestep and consists of three operations. First, the nonconformity score $s(X_t, Y_t; \theta_t)$ is evaluated for the current observation, requiring $O(r)$ time where $r$ is the number of facets of the convex hull. Second, the inverse quantile $\beta_t$ is computed by evaluating the score against the rolling window $\mathcal{D}_t$, requiring $O(W)$ time. Third, the $k$ expert weights are updated via exponential reweighting and each expert's quantile $\alpha_t^i$ is updated, requiring $O(k)$ time. The total per-timestep cost of the fast path is therefore $O(r + W + k)$.

\textbf{Score function update:}
The score function update executes every $t_s$ timesteps and consists of three stages: reweighting, score optimization, and replay.

\textit{Reweighting:} Computing the adaptive weights $\omega_\tau$ for all history buffer entries requires evaluating the exponential decay across $k$ experts for each of the $W$ buffered points, yielding $O(kW)$.

\textit{Score optimization:} This stage comprises density estimation and shape template fitting. The weighted KDE is constructed from $W$ samples in $O(Wd)$ time. Drawing $M$ Monte Carlo samples from the KDE costs $O(Md)$. Evaluating the KDE at each Monte Carlo sample requires summing over all $W$ kernel centers, yielding $O(MWd)$. Selecting the top $1-\alpha$ fraction by density requires sorting in $O(M \log M)$. Finally, the QuickHull algorithm fits the convex hull to the $n$ retained high-density points in $O(n \log n)$ expected time. The dominant cost of this stage is the KDE evaluation at $O(MWd)$.

\textit{Replay:} The replay mechanism re-initializes the DtACI expert chain and processes the $W$ most recent observations under the updated score function. Each replayed observation requires $O(r + k)$ for score evaluation and expert updates, yielding a total replay cost of $O(W(r+k))$.

The total cost of the score update is $O(MWd + M \log M + n \log n + W(r+k))$, dominated by the KDE evaluation term $O(MWd)$.

\textbf{Amortized cost:}
Since the score function update is executed every $t_s$ timesteps, its per timestep cost amortizes to: $O\!\left(MWd/t_s\right)$. Adding the threshold update cost yields the amortized per-timestep cost of AdaptNC:
\[
O\!\left(
W + k + MWd/t_s
\right)
\]
\end{proof}
\clearpage
    \section{Theoretical Details and Proofs}\label{appendix:theory}

\subsection{Background ACI and DtACI Guarantees}
We present the theorem that yields \cref{rem:big-o-regret}
\begin{restatable}[Theorem 4, \citep{gibbs_conformal_2024}]{theorem}{dtaciregret}
\label{thm:dtaci_regret}
Let $\gamma_{max} \coloneqq \max_{1\leq i\leq k} \gamma_t$ and assume that $\gamma_1 < \gamma_2 < ... < \gamma_k$ where $\gamma_{i+1}/\gamma_i \leq 2$ for all $1 < i \leq k$, and that $\ell(\beta_t,\theta)\coloneqq \alpha(\beta_t -\theta) - \min\{0,\beta_t-\theta\}$ Then, for any window $W = [r,s] \in [T]$ and any sequence $\alpha^*_r, \ldots, \alpha^*_s \in [0,1]$,
\begin{equation}
    \begin{split}
        \frac{1}{|W|} \sum_{t=r}^s \ell((\beta_t,\bar{\alpha}_t)) - \frac{1}{|W|} \sum_{t=r}^s \ell((\beta_t,\alpha^*_t)) \leq \\
        \frac{\log(k/\sigma)+2\sigma|W|}{\eta |W|}
        +\frac{\eta}{|W|}\sum_{t=r}^{s}\mathbb{E}\!\big[\ell(\beta_t,\alpha_t)^2\big]
        + \\ 4(1+\gamma_{\max})^{2}\max\!\left\{
        \sqrt{\frac{\sum_{t=r+1}^{s}\left|\alpha_t^{*}-\alpha_{t-1}^{*}\right|+1}{|W|}},
        \gamma_{1}
        \right\}.
    \end{split}
\end{equation}
\end{restatable}
\begin{remark}\label{rem:big-o-regret}
The bound of the error presented here can be written in simpler form for specific choices of $\sigma$ and $\eta$. Specifically, $\eta = \sqrt{\frac{\log(2k|W|)+1}{\sum_{t=r}^{s}\E[\ell(\beta_t,\alpha_t)^2]}}$ and $\sigma=1/(2|W|)$
\begin{equation}
    \begin{split}
        \frac{1}{|W|} \sum_{t=r}^s \ell((\beta_t,\bar{\alpha}_t)) - \frac{1}{|W|} \sum_{t=r}^s \ell((\beta_t,\alpha^*_t)) \leq
        O\left(\sqrt{\frac{\log(|W|)}{|W|}}\right)+ O\left(
        \sqrt{\frac{\sum_{t=r+1}^{s}\left|\alpha_t^{*}-\alpha_{t-1}^{*}\right|+1}{|W|}}\right).
    \end{split}
\end{equation}
This derivation is due to \citet{gibbs_conformal_2024}
\end{remark}

\begin{theorem}[Theorem 6, \citep{gibbs_conformal_2024}]
\label{thm:dtaci_longterm} Consider a modified version of \cref{alg:dtaci} in which on iteration $t$ the parameters $\eta$ and $\sigma$ are replaced by values $\eta_t$ and $\sigma_t$. Let $\gamma_{\min} := \min_i \gamma_i$,$\gamma_{\max} := \max_i \gamma_i$, and that $\lim_{t\to\infty} \eta_t = \lim_{t\to\infty} \sigma_t = 0$. Then,
\[
\lim_{T\to\infty} \frac{1}{T} \sum_{t=1}^T \mathrm{err}_t \overset{a.s.}{=} \alpha
\]
where the expectation is over the randomness in the randomized variant of DtACI and the data $\beta_1, \dots, \beta_T$ can be viewed as fixed.
\end{theorem}

\subsection{HDR and Score-Stability Proofs}\label{app:score-stability-proofs}

\begin{assumption}[Regularity for HDR stability]
\label{assump:hdr_regular}
Let $\mathcal Z_t$ denote the effective weighted residual distribution used by
AdaptNC at score-update time $t$, and let $f_t$ denote its density. Let
$\mathcal Z_\star$ be the limiting weighted residual distribution with density
$f_\star$.

\begin{enumerate}
    \item The residual-distribution convergence is strong enough to imply
    uniform density convergence:
    \[
        \|f_t-f_\star\|_\infty\to0.
    \]

    \item The KDE kernel, bandwidth sequence, and weighted effective sample size
    satisfy the standard conditions for uniform KDE consistency:
    \[
        \|\widehat f_t-f_t\|_\infty\xrightarrow{p}0.
    \]

    \item The density-score quantile at level $\alpha$ is regular: if
    \[
        G_\star(u):=
        \mathbb P_{Z\sim\mathcal Z_\star}(f_\star(Z)\le u),
    \]
    then $G_\star$ is continuous and strictly increasing in a neighborhood of
    $\tau_\star:=G_\star^{-1}(\alpha)$.

    \item The limiting HDR boundary is regular: there exist constants
    $c>0$ and $\varepsilon_0>0$ such that
    \[
        \|\nabla f_\star(z)\|_2\ge c
        \qquad
        \text{whenever }
        |f_\star(z)-\tau_\star|\le\varepsilon_0.
    \]

    \item The relevant HDRs are nonempty compact subsets of $\mathbb R^p$.
    \item There exists \(L_{\mathrm{vol}}<\infty\) such that for all \(x\), all thresholds \(q\) in the relevant range, and all sufficiently small \(\varepsilon>0\)
    \[ \operatorname{Vol} \{y : q-\varepsilon \le s_\star(x,y)\le q+\varepsilon\}
    \le L_{\mathrm{vol}}\varepsilon.\]
\end{enumerate}
\end{assumption}

\begin{lemma}[Hausdorff consistency of the MCKDE high-density region]
\label{lem:mckde_hausdorff_consistency}
Let
\[
    R_\star := \{z\in\Omega:f_\star(z)\ge\tau_\star\}
\]
be the limiting $(1-\alpha)$ high-density residual region, where $\tau_\star$
is chosen so that
\[
    \int_{R_\star} f_\star(z)\,dz = \int_{\Omega}\mathbf 1\{f_\star(z)\ge\tau_\star\}f_\star(z)\,dz = 1-\alpha.
\]

At score-update time $t$, let $\widehat f_t$ be the KDE formed from $N_t$
weighted residual samples, and let $\widehat\tau_t$ be the Monte Carlo estimate
of the density threshold obtained from $M_t$ samples. Define the MCKDE estimated
high-density region
\[
    \widehat R_t := \{z\in\Omega:\widehat f_t(z)\ge\widehat\tau_t\}.
\]
Under Assumptions~\ref{assump:weighted_residual_stability}
and~\ref{assump:hdr_regular},
\[
    d_H(\widehat R_m,R_\star)\xrightarrow{p}0.
\]
\end{lemma}
\begin{proof}
By the assumed stability of the weighted residual distributions and the assumed
relation between $D_\phi$ and the residual density class,
\[
    \|f_t-f_\star\|_\infty\to0.
\]
By uniform KDE consistency,
\[
    \|\widehat f_t-f_t\|_\infty\xrightarrow{p}0.
\]
Therefore,
\[
    \|\widehat f_t-f_\star\|_\infty\xrightarrow{p}0.
\]

The Monte Carlo threshold $\widehat\tau_t$ is the empirical quantile of the
density scores $\widehat f_t(Z)$, where $Z$ is sampled from the KDE. Since
$M_t\to\infty$, the empirical density-score CDF converges uniformly to its
population counterpart, by the Glivenko--Cantelli theorem. Together with the regularity of the density-score
quantile and the convergence $\widehat f_t\to f_\star$, this gives
\[
    |\widehat\tau_t-\tau_\star|\xrightarrow{p}0.
\]

Define
\[
    \varepsilon_t := \|\widehat f_t-f_\star\|_\infty + |\widehat\tau_t-\tau_\star|.
\]
Then $\varepsilon_t\xrightarrow{p}0$.

If $z\in\widehat R_t$, then $\widehat f_t(z)\ge\widehat\tau_t$, and hence
\[
    f_\star(z)
    \ge
    \widehat f_t(z)-\|\widehat f_t-f_\star\|_\infty
    \ge
    \widehat\tau_t-\|\widehat f_t-f_\star\|_\infty
    \ge
    \tau_\star-\varepsilon_t.
\]
Therefore,
\[
    \widehat R_t \subseteq \{z\in\Omega:f_\star(z)\ge\tau_\star-\varepsilon_t\}.
\]

Similarly, if $z\in R_\star$, then $f_\star(z)\ge\tau_\star$. Points in
$R_\star$ with $f_\star(z)\ge\tau_\star+\varepsilon_t$ satisfy
\[
    \widehat f_t(z)
    \ge
    f_\star(z)-\|\widehat f_t-f_\star\|_\infty
    \ge
    \tau_\star+\varepsilon_t-\|\widehat f_t-f_\star\|_\infty
    \ge
    \widehat\tau_t,
\]
and therefore belong to $\widehat R_t$. Thus, any point of $R_\star$ that may
fail to belong to $\widehat R_t$ must lie in the shrinking level band
\[
    \{z\in\Omega: |f_\star(z)-\tau_\star|\le\varepsilon_t\}.
\]

By the regular-boundary assumption, these level bands shrink to
$\partial R_\star$ in Hausdorff distance. More precisely, for all sufficiently
small $\varepsilon$, every point satisfying
\[
    |f_\star(z)-\tau_\star|\le\varepsilon
\]
lies within distance $O(\varepsilon)$ of the boundary
$\{z:f_\star(z)=\tau_\star\}$. Consequently,
\[
    d_H(\widehat R_t,R_\star) \le C\varepsilon_t
\]
for some constant $C>0$ and all sufficiently large $t$, with probability tending
to one. Since $\varepsilon_t\xrightarrow{p}0$, we conclude that
\[
    d_H(\widehat R_t,R_\star)\xrightarrow{p}0.
\]
\end{proof}

\begin{lemma}[Convex-hull stability]
\label{lem:convex_hull_stability}
Let $A_m,A_\star\subset\mathbb R^p$ be nonempty compact sets.
\[
    d_H(A_m,A_\star)\to0 \implies d_H(\operatorname{conv}(A_m),\operatorname{conv}(A_\star))\to0.
\]
\end{lemma}

\begin{proof}
Let $\rho_m=d_H(A_m,A_\star)$. For any
$x\in\operatorname{conv}(A_m)$, write
\[
    x=\sum_{\ell=1}^L\lambda_\ell a_\ell,
    \qquad
    a_\ell\in A_m,\quad
    \lambda_\ell\ge0,\quad
    \sum_{\ell=1}^L\lambda_\ell=1.
\]
For each $a_\ell$, choose $a_\ell^\star\in A_\star$ such that
$\|a_\ell-a_\ell^\star\|_2\le\rho_m$. Then
\[
    x^\star:=\sum_{\ell=1}^L\lambda_\ell a_\ell^\star
    \in\operatorname{conv}(A_\star),
\]
and
\[
    \|x-x^\star\|_2 \le \sum_{\ell=1}^L\lambda_\ell\|a_\ell-a_\ell^\star\|_2 \le \rho_m.
\] where the last inequality follows from the definition of the Hausdorff distance.
Thus every point of $\operatorname{conv}(A_m)$ lies within $\rho_m$ of
$\operatorname{conv}(A_\star)$. The reverse inclusion is identical. Hence
\[
    d_H(\operatorname{conv}(A_m),\operatorname{conv}(A_\star)) \le \rho_m.
\]
\end{proof}
\begin{corollary}
    In the setting of \cref{lem:convex_hull_stability},
    \[
    d_H(\operatorname{conv}(A_m),\operatorname{conv}(A_\star)) \le d_H(A_m,A_\star).
    \]
\end{corollary}

Because halfspace representations of a convex set are not unique, we state
score stability for the canonical support-function score induced by the fitted
convex template. Equivalently, the implemented halfspace score should be
normalized to this canonical representation before comparing scores across
updates. Our implementation, which relies on QuickHull \cite{barber_quickhull_1996}, performs this normalization.

\begin{lemma}[Score stability from hull stability]
\label{lem:score_stability_from_hull}
For a compact convex set $K\subset\mathbb R^p$, define
\[
    s_K(x,y) := \sup_{\|v\|_2=1}\{v^\top(y-h(x))-H_K(v)\}, \qquad H_K(v):=\sup_{z\in K}v^\top z.
\]
If
\[
    d_H(K_t,K_\star)\to0,
\]
then
\[
    \sup_{(x,y)\in\mathcal A}|s_{K_t}(x,y)-s_{K_\star}(x,y)| \to0
\]
for a set $\mathcal A\subseteq\mathcal X\times\mathcal Y$.
\end{lemma}

\begin{proof}
For any residual $z=y-h(x)$,
\[
\begin{aligned}
|s_{K_m}(x,y)-s_{K_\star}(x,y)|
&=
\left|
\sup_{\|v\|_2=1}
\{v^\top z-h_{K_m}(v)\}
-
\sup_{\|v\|_2=1}
\{v^\top z-h_{K_\star}(v)\}
\right| \\
&\le
\sup_{\|v\|_2=1}
|H_{K_m}(v)-H_{K_\star}(v)|.
\end{aligned}
\]
For compact convex sets,
\[
    \sup_{\|v\|_2=1}|H_{K_m}(v)-H_{K_\star}(v)| = d_H(K_m,K_\star).
\]
Therefore,
\[
    |s_{K_m}(x,y)-s_{K_\star}(x,y)| \le d_H(K_m,K_\star).
\]
Taking the supremum over $(x,y)\in\mathcal A$ gives the result.
\end{proof}

\scorefuncstability*
\begin{proof}\label{proof:score-func-stability}
By Lemma~\ref{lem:mckde_hausdorff_consistency},
\[
    d_H(\widehat R_m,R_\star)\xrightarrow{p}0.
\]
By Lemma~\ref{lem:convex_hull_stability},
\[
    d_H(K_m,K_\star) = d_H(\operatorname{conv}(\widehat R_m),\operatorname{conv}(R_\star)) \xrightarrow{p}0.
\]
By Lemma~\ref{lem:score_stability_from_hull},
\[
    \sup_{(x,y)\in\mathcal A}|s_{K_m}(x,y)-s_{K_\star}(x,y)| \xrightarrow{p}0.
\]
Finally, by the triangle inequality,
\[
\begin{aligned}
&\sup_{(x,y)\in\mathcal A}|s_{K_{m+1}}(x,y)-s_{K_m}(x,y)| \\
&\qquad\le \sup_{(x,y)\in\mathcal A}|s_{K_{m+1}}(x,y)-s_{K_\star}(x,y)|
    + \sup_{(x,y)\in\mathcal A}|s_{K_m}(x,y)-s_{K_\star}(x,y)|.
\end{aligned}
\]
Both terms converge to zero in probability, so the successive score functions
are asymptotically stable.
\end{proof}

\begin{proof}[Proof of \cref{cor:prediction_region_size_stability}]
Suppose
\[
    \sup_{(x,y)\in\mathcal A}|s_t(x,y)-s_\star(x,y)| \le \delta_t.
\]
Under Assumption~\ref{assump:hdr_regular}, for every fixed threshold \(q\),
\[
    \left|
    \operatorname{Vol}(C_t(x;q))
    -
    \operatorname{Vol}(C_\star(x;q))
    \right|
    \le
    L_{\mathrm{vol}}\delta_t.
\]
Since \(\delta_t\to0\) in probability by \cref{thm:score_func_stability}, the prediction-region volume stabilizes in probability.
\end{proof}

\subsection{Long-Run Coverage under Replay}\label{app:coverage-replay-proofs}

\begin{proof}[Proof of \cref{thm:per_expert_longrun_coverage}]\label{proof:per_expert_longrun_coverage}
Expanding the perturbed recursion gives
\[
    \alpha_{T+1}^i = \alpha_1^i + \gamma_i\sum_{t=1}^T(\alpha-\err_t^i) + \sum_{t=1}^T\Delta_t^i.
\]
Rearranging,
\[
    \frac{1}{T}\sum_{t=1}^T \err_t^i-\alpha
    =
    \frac{\alpha_1^i-\alpha_{T+1}^i}{T\gamma_i} + \frac{1}{T\gamma_i}\sum_{t=1}^T\Delta_t^i.
\]
The first term vanishes because the expert states are bounded. The second term vanishes
by the assumed sublinear cumulative perturbation.
\[
\left|
\frac{1}{T\gamma_i}\sum_{t=1}^T\Delta_t^i
\right|
\le
\frac{1}{T\gamma_i}\sum_{t=1}^T|\Delta_t^i|
\to0.
\]
Hence:
\[
    \frac{1}{T}\sum_{t=1}^T \err_t^i \to \alpha .
\]
\end{proof}

\begin{proof}[Proof of \cref{thm:perturbation_to_zero_avg}]
For the analysis, define the pre-update replay trajectory as the counterfactual
trajectory obtained by replaying \(\mathcal R_m\) using \(s_m^-\), and the
post-update replay trajectory as the trajectory obtained by replaying the same
window using \(s_m^+\), both initialized from the same expert state. Let
\[
    e_{m,r}^{i,-}
    :=
    \mathbf 1\{s_m^-(X_r,Y_r)>q_{m,r}^{i,-}\},
    \qquad
    e_{m,r}^{i,+}
    :=
    \mathbf 1\{s_m^+(X_r,Y_r)>q_{m,r}^{i,+}\}.
\]
At non-update times, \(\Delta_t^i=0\). At update time \(t_m\), the difference
between the replayed post-update expert state and the corresponding
pre-update expert state is controlled by the number of replayed observations
whose coverage indicators change:
\[
    |\Delta_{t_m}^i| \le \gamma_i\sum_{r\in\mathcal R_m}|e_{m,r}^{i,+}-e_{m,r}^{i,-}|.
\]
Now observe that the indicators can differ only if the replayed score lies
near the pre-update threshold. Indeed, if
\[
    |s_m^-(X_r,Y_r)-q_{m,r}^{i,-}| > (1+L_q)\varepsilon_m,
\]
then
\[
    |s_m^+(X_r,Y_r)-s_m^-(X_r,Y_r)| \le \varepsilon_m
\]
and
\[
    |q_{m,r}^{i,+}-q_{m,r}^{i,-}| \le L_q\varepsilon_m
\]
together imply that \(s_m^+(X_r,Y_r)\) and \(s_m^-(X_r,Y_r)\) remain on the
same side of their respective thresholds. Hence
\[
    |e_{m,r}^{i,+}-e_{m,r}^{i,-}|
    \le
    \mathbf 1
    \left\{
        |s_m^-(X_r,Y_r)-q_{m,r}^{i,-}|
        \le
        (1+L_q)\varepsilon_m
    \right\}.
\]
Therefore
\[
    \frac1T\sum_{t=1}^T|\Delta_t^i|
    \le
    \frac{\gamma_i}{T}
    \sum_{m:t_m\le T}
    \sum_{r\in\mathcal R_m}
    \mathbf 1
    \left\{
        |s_m^-(X_r,Y_r)-q_{m,r}^{i,-}|
        \le
        (1+L_q)\varepsilon_m
    \right\}.
\]
The assumed averaged boundary condition implies that the right-hand side
converges to zero. The final coverage claim follows from the perturbed
single-expert ACI guarantee.
\end{proof}

\subsection{Coverage Shock from Score Updates}\label{app:coverage-shock}

{\allowdisplaybreaks
\begin{proposition}\label{prop:dset-change-bound}
    Let $\mathcal{D}_t$ and $\mathcal{D}_{t+1}$ be the joint distributions of inputs and observations $(X_t,Y_t)$ at times $t$ and $t+1$. Assume that the total variation distance ($\TV$) between the two distributions is bounded $\TV(\mathcal{D}_t,\mathcal{D}_{t+1}) \leq \delta$.
    For a fixed NCSF $s(\cdot)$, and a fixed score parameter $\theta$ let $S_{\theta,t}$ be the random variable $s(X,Y,\theta)$ where $(X,Y)\sim \mathcal{D}_t$, and $S_{\theta,t+1}$ be the random variable where $(X,Y)\sim \mathcal{D}_{t+1}$ Then, the $\TV$ distance between the score distributions will be less than $\delta$:
    \(\TV(S_{\theta,t},S_{\theta,t+1}) \leq \delta\)
\end{proposition}
\begin{proof}
    This follows from the Data Processing Inequality \citep[Theorem 7.4]{polyanskiy_information_2025} with NCSF $s(\cdot;\theta)$ as channel, and $f$-divergence as $\TV$ distance.
\end{proof}
}
\clearpage
    \section{Experiments}\label{appendix:experiments}

\subsection{Gaussian Mixture Model: A Minimal Example of Coverage Shock}
\label{appendix:gmm}
\begin{wrapfigure}{r}{0.38\columnwidth}
    \centering
    \includegraphics[width=\linewidth]{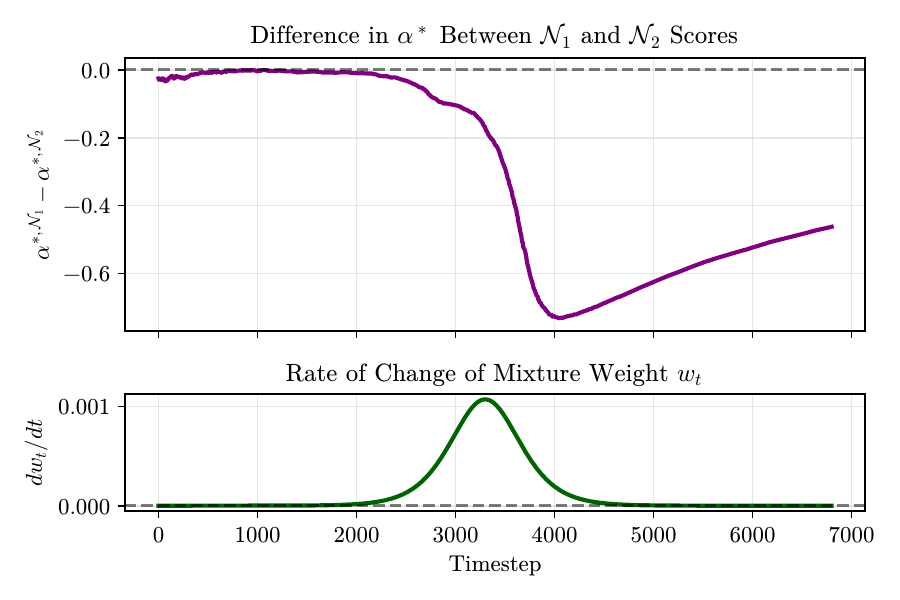}
    % \vspace{-5pt}
    \caption{Score-induced coverage shock in GMM example. The top panel shows the gap between history-relative effective miscoverage levels induced by the two diagnostic scores, while the bottom panel shows the rate of change of the mixture weight. Despite smooth distributional changes, altering the score can induce large shifts in the calibration state that DtACI must track.}
    \label{fig:adapt_nc_alpha_star}
    % \vspace{2em}
\end{wrapfigure}

The replay mechanism in AdaptNC is motivated by a simple failure mode: the
state of an adaptive thresholding algorithm is tied to the nonconformity score
used to map past residuals into ranks and coverage errors. When the score
function changes, the same recent residuals can induce a different score
distribution and therefore a different calibration state. Thus, even if the
underlying residual distribution changes smoothly, a score update can produce a
transient mismatch between the inherited threshold state and the updated score.
We refer to this transient mismatch as \emph{coverage shock}. The following
Gaussian mixture example isolates this effect in a setting where the only source
of nonstationarity is a smooth change in mixture weights.

\paragraph{Data-generating process.}
At time $t$, a residual sample $z_t\in\mathbb R^2$ is drawn from a two-component
Gaussian mixture
\begin{equation}
    p_t(z)=(1-w_t)p_1(z)+w_t p_2(z),
\end{equation}
where $w_t$ varies smoothly over time. Thus the stream begins in a regime
dominated by $\mathcal N_1$ and gradually transitions to a regime dominated by
$\mathcal N_2$. The two components are
\begin{equation*}
    \mathcal{N}_1 = \mathcal{N}\left(
    \begin{bmatrix}1.0 & -1.2\end{bmatrix},
    \begin{bmatrix}
        1.2 & 0.6 \\ 0.6 & 0.9
    \end{bmatrix}
    \right), \qquad
    \mathcal{N}_2 = \mathcal{N}\left(
    \begin{bmatrix}-1.0 & 1.2\end{bmatrix},
    \begin{bmatrix}
        0.8 & -0.3 \\ -0.3 & 1.1
    \end{bmatrix}
    \right).
\end{equation*}
The resulting stream is shown in \cref{fig:appendix-gmm-stream}.

\begin{figure}[h]
    \centering
    \begin{minipage}[t]{0.49\linewidth}
        \centering
        \includegraphics[width=\linewidth]{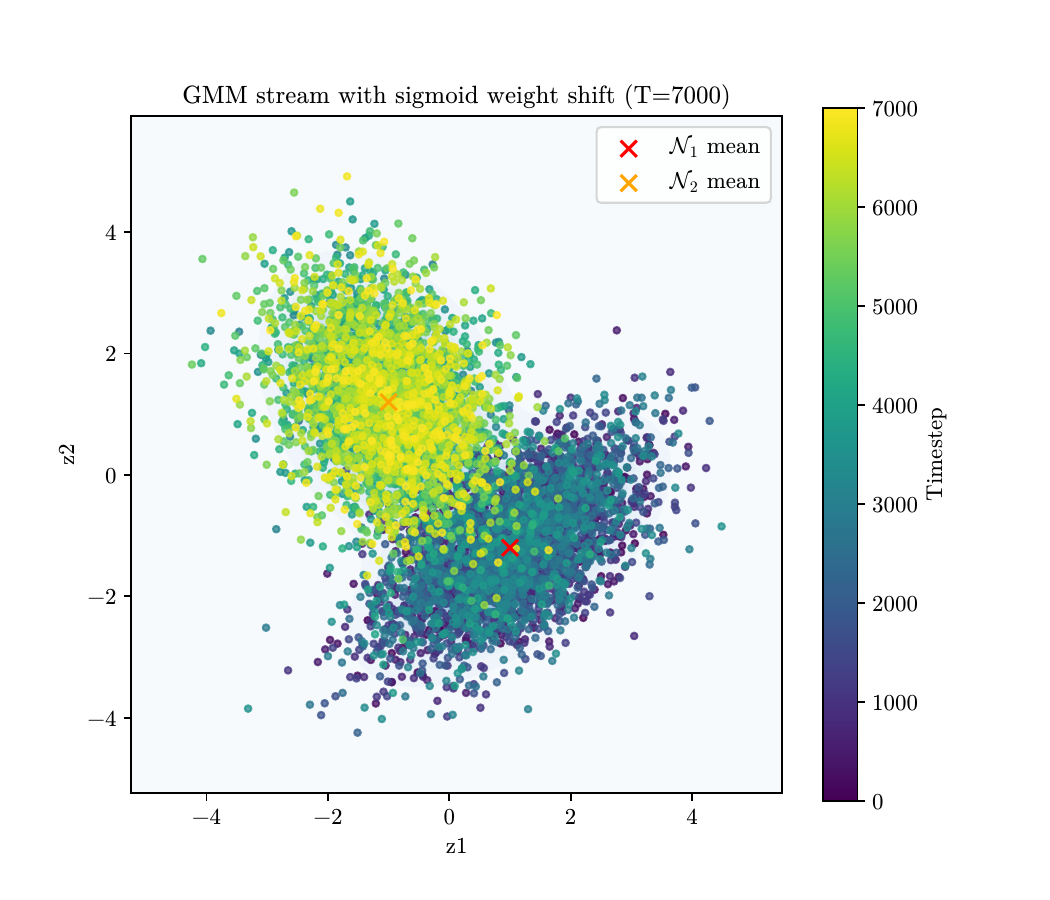}
        \captionof{figure}{The 7000 samples drawn from the shifting Gaussian mixture stream. Color indicates time, showing the transition from the $\mathcal N_1$-dominated regime to the $\mathcal N_2$-dominated regime.}
        \label{fig:appendix-gmm-stream}
    \end{minipage}\hfill
    \begin{minipage}[t]{0.49\linewidth}
        \centering
        \includegraphics[width=\linewidth]{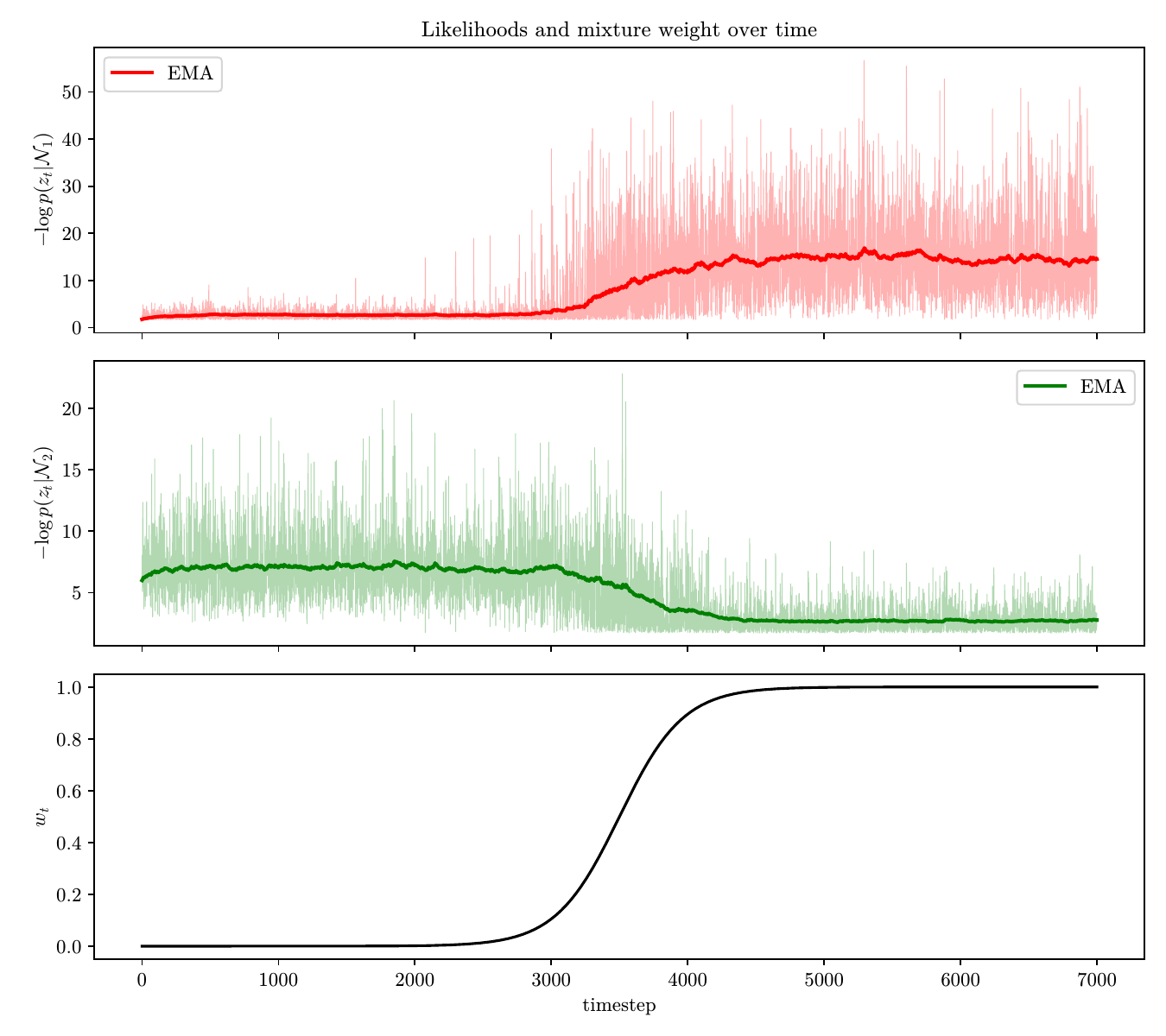}
        \captionof{figure}{Negative log-likelihood scores under the two component scores as the mixture weight changes. The smooth mixture transition induces a systematic change in the score distribution.}
        \label{fig:appendix-gmm-likelihood}
    \end{minipage}
\end{figure}

\paragraph{Score-dependent reference levels.}
We compare two fixed nonconformity scores, each matched to one component:
\begin{equation}
    s_a(z)=-\log p_a(z), \qquad a\in\{1,2\}.
\end{equation}
These scores are used as interpretable diagnostic endpoints: $s_1$ is matched to
the early regime and $s_2$ is matched to the late regime. AdaptNC does not
literally switch between these two scores; rather, the pair illustrates how a
change in score geometry can change the calibration state even on the same
residual stream.

For each score $s_a$, let
\begin{equation}
q_{a,t}^{\mathrm{mix}}
=
\inf\left\{
q:
\Pr_{Z\sim p_t}\left[s_a(Z)\le q\right]\ge 1-\alpha
\right\}
\end{equation}
be the nominal $(1-\alpha)$ score quantile under the instantaneous mixture
$p_t$. In the experiment, this quantity is estimated by Monte Carlo from the
current mixture distribution. We then evaluate where this instantaneous quantile
falls relative to the accumulated score history,
\begin{equation}
\widehat F^{\mathrm{hist}}_{a,t}(u)
=
\frac{1}{t+1}
\sum_{\tau=0}^{t}
\mathds{1}\{s_a(z_\tau)\le u\},
\qquad
\widehat \alpha_{a,t}^{\star}
=
1-\widehat F^{\mathrm{hist}}_{a,t}\!
\left(q_{a,t}^{\mathrm{mix}}\right).
\end{equation}
The quantity $\widehat \alpha_{a,t}^{\star}$ is a diagnostic effective
miscoverage level: it is the miscoverage rate that would make the accumulated
score history place its $(1-\widehat \alpha_{a,t}^{\star})$ quantile at the
instantaneous mixture quantile. Because both $q_{a,t}^{\mathrm{mix}}$ and
$\widehat F^{\mathrm{hist}}_{a,t}$ depend on the score map, changing from
$s_1$ to $s_2$ can induce a large jump in $\widehat \alpha_{a,t}^{\star}$ even
when $w_t$ changes smoothly.

\begin{figure}[t]
    \centering
    \begin{minipage}[t]{0.475\linewidth}
        \includegraphics[width=\linewidth]{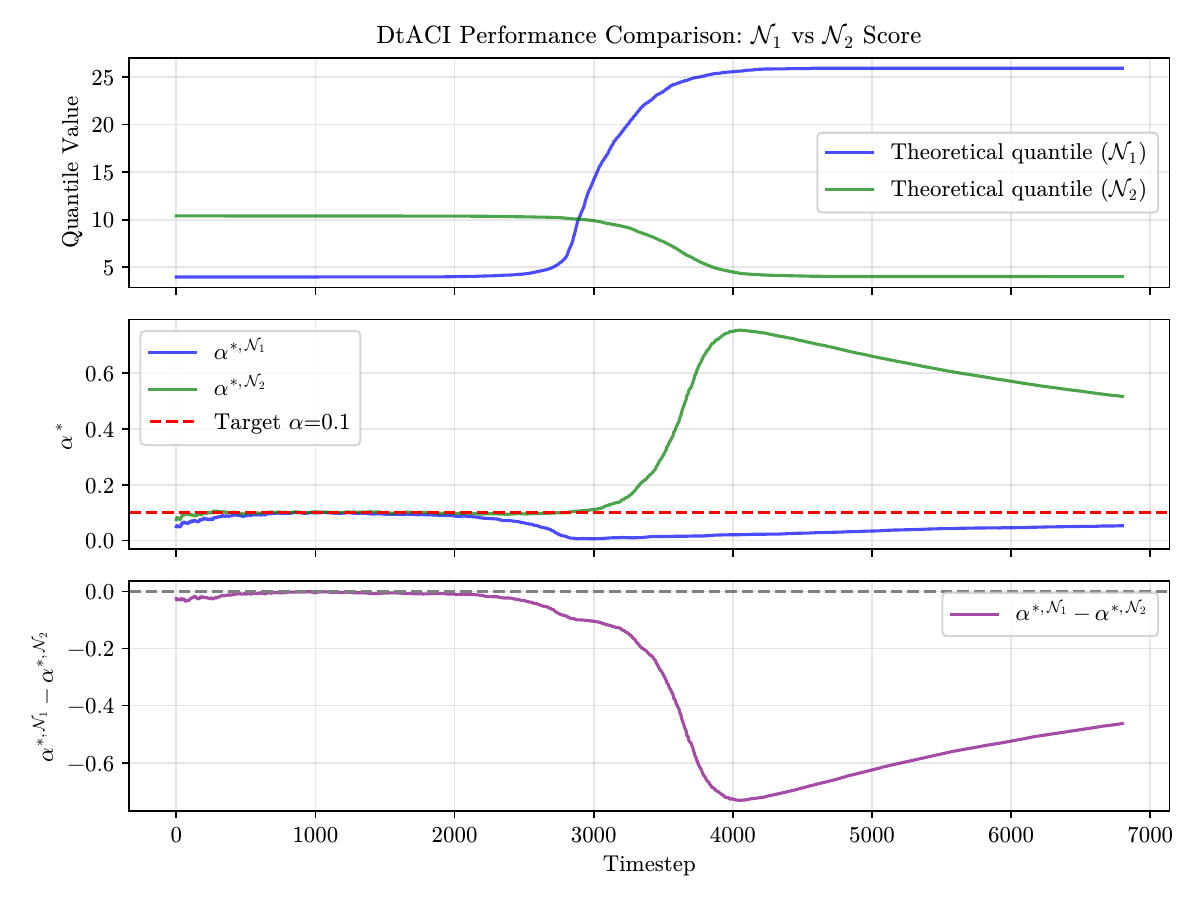}
        \caption{Coverage shock in the shifting GMM. \textbf{Top:} The nominal $(1-\alpha)$ score quantiles for the two diagnostic scores evolve in opposite directions as the dominant mixture component changes. \textbf{Middle:} The corresponding effective miscoverage levels $\widehat\alpha_{a,t}^{\star}$ deviate from the nominal target $\alpha=0.1$ and from one another. \textbf{Bottom:} The gap between these levels quantifies the change in calibration state induced by changing scores.}
        \label{fig:appendix-gmm-alpha-star}
    \end{minipage}\hfill
    \begin{minipage}[t]{0.475\linewidth}
        \includegraphics[width=\linewidth]{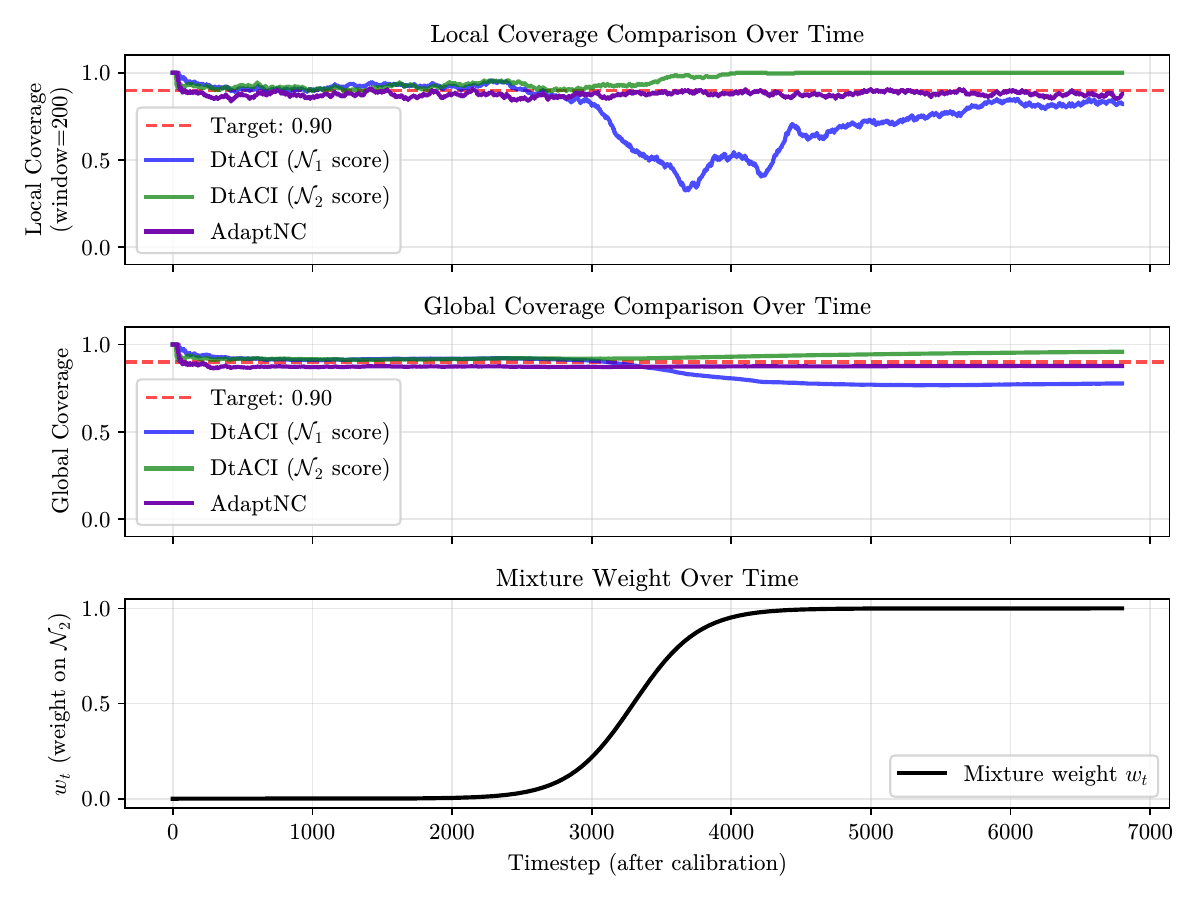}
        \caption{Coverage behavior under the shifting GMM. The fixed $s_1$ baseline undercovers after the mixture shifts, while the fixed $s_2$ baseline is conservative. AdaptNC remains closer to the target than the under-covering fixed score and avoids the persistent conservatism of the late-regime score by adapting its score and replaying recent observations.}
        \label{fig:appendix-gmm-local-coverage}
    \end{minipage}
\end{figure}

\paragraph{Why this motivates replay.}
\Cref{fig:appendix-gmm-alpha-star} shows that a smooth change in mixture weight
can still create a large separation between the calibration reference levels
associated with different scores. If a score-adaptive method updates its score
without recalibrating the threshold state, DtACI continues from a state shaped by
old score observations while receiving future observations through a new score.
In the regret bound of \cref{rem:big-o-regret}, this appears as a large
$|\alpha_t^\star-\alpha_{t-1}^\star|$ term, i.e., coverage shock.

Replay mitigates this effect by reconstructing the recent DtACI state under the
updated score. Instead of applying the new score to future observations while
retaining a threshold state learned from the old score, AdaptNC rescores the
recent window, recomputes the corresponding ranks and coverage errors, and reruns
the threshold updates. In the GMM example, this replaces a counterfactual abrupt
jump between the diagnostic endpoint scores with a recalibration step under the
updated score used by AdaptNC.

This behavior is reflected in \cref{fig:appendix-gmm-local-coverage}. When a
fixed score becomes mismatched with the dominant mixture component, DtACI either
undercovers or becomes conservative. AdaptNC adapts the score and uses replay to
realign the threshold state, yielding more stable near-target coverage through
the transition. The example is intentionally simple: its purpose is not to model
a robotics task, but to isolate the mechanism by which score changes create
coverage shock and to motivate replay as the corresponding state-realignment
step.
\clearpage

\subsection{Indoor Localization}\label{appendix:indoor-localization}

This section provides a complete and reproducible description of the environment dynamics, wireless observation model, trajectory predictor, and evaluation policy used in all experiments.

\subsubsection{Environment Dynamics}

We consider a planar indoor localization environment with discrete-time dynamics and sampling period $\Delta t = 0.1$ s. The true system state at time $t$ is
\[
\mathbf{s}_t = [x_t, y_t, v_{x,t}, v_{y,t}]^\top \in \mathbb{R}^4,
\]
where $(x_t, y_t)$ denotes the agent position and $(v_{x,t}, v_{y,t})$ denotes velocity.

Given an acceleration control input $\mathbf{a}_t = [a_{x,t}, a_{y,t}]^\top$, the dynamics evolve according to
\[
\begin{aligned}
x_{t+1} &= x_t + \Delta t \, v_{x,t}, \\
y_{t+1} &= y_t + \Delta t \, v_{y,t}, \\
v_{x,t+1} &= v_{x,t} + \Delta t \, (a_{x,t} + \varepsilon_{x,t}), \\
v_{y,t+1} &= v_{y,t} + \Delta t \, (a_{y,t} + \varepsilon_{y,t}),
\end{aligned}
\]
where $\varepsilon_{x,t}, \varepsilon_{y,t} \sim \mathcal{N}(0, \sigma_{\text{proc}}^2)$ model additive process noise on acceleration. We use $\sigma_{\text{proc}} = 0.02$.

The agent operates within a bounded workspace defined by $x,y \in [-6,6]$. Reflective boundary conditions are enforced by reversing the corresponding velocity component whenever the agent reaches the workspace boundary.

\subsubsection{Access Point Geometry}

The environment contains four fixed access points located at
\[
(-5,-5), \; (5,-5), \; (5,5), \; (-5,5).
\]
These access points are static and known to the environment simulator and predictor.

\subsubsection{RSSI Observation Model}

At each time step, the agent observes a received signal strength indicator vector
\[
\mathbf{o}_t = [\text{RSSI}_t^{(1)}, \ldots, \text{RSSI}_t^{(4)}]^\top \in \mathbb{R}^4.
\]

Each RSSI measurement is generated as the sum of path loss, shadowing, and small-scale fading components, expressed in decibels:
\[
\text{RSSI}_t^{(i)} = P_0 - 10 n \log_{10}(d_t^{(i)}) + S_t^{(i)} + F_t^{(i)}.
\]

\paragraph{Path Loss.}
The path loss component follows the standard log-distance path loss model~\cite{rappaport2010wireless}. Here, $P_0 = -30$ dB denotes the reference received power at unit distance, $n = 2.2$ is the path loss exponent, and
\[
d_t^{(i)} = \| [x_t, y_t] - \mathbf{a}_i \|_2
\]
is the Euclidean distance between the agent and access point $i$. Distances are lower bounded by $10^{-2}$ m to ensure numerical stability. The path loss exponent controls the rate of signal attenuation with distance and reflects indoor propagation characteristics.

\paragraph{Shadowing.}
Large-scale shadowing is modeled as a temporally correlated log-normal process~\cite{gudmundson1991correlation}. For each access point $i$, the shadowing term evolves according to a first-order Gauss--Markov model:
\[
S_t^{(i)} = \rho S_{t-1}^{(i)} + \sqrt{1 - \rho^2} \, \sigma_{\text{sh}} \, \xi_t^{(i)},
\quad
\xi_t^{(i)} \sim \mathcal{N}(0,1),
\]
where $\sigma_{\text{sh}} = 4.0$ dB controls the shadowing variance and $\rho = 0.97$ determines temporal correlation. Shadowing states are initialized to zero at environment reset and evolve independently across access points.

\paragraph{Small-Scale Fading.}
Small-scale fading is modeled using a Rayleigh fading process generated via a sum-of-sinusoids~\cite{jakes1994microwave} method. For each access point, a complex channel coefficient $h_t^{(i)}$ is generated, and the fading contribution is computed as
\[
F_t^{(i)} = 10 \log_{10}(|h_t^{(i)}|^2 + \epsilon),
\]
with $\epsilon = 10^{-12}$ for numerical stability.

The Doppler frequency governing temporal fading correlation is velocity-dependent:
\[
f_{d,t} = \frac{\|\mathbf{v}_t\|_2}{\lambda},
\quad
\lambda = \frac{c}{f_c},
\]
where $f_c = 2.4$ GHz is the carrier frequency and $c$ is the speed of light. Fading processes are independent across access points. The Rayleigh distribution of $|h_t^{(i)}|$ follows from the assumption of rich multipath with no dominant 
line-of-sight component~\cite{clarke1968statistical, rappaport2010wireless}.

\subsubsection{Trajectory Predictor}
\label{sec:appendix_predictor}

We employ a lightweight model-based trajectory predictor that estimates the agent position from RSSI measurements using a simplified propagation model. The predictor is mismatched with the true environment observation model and does not explicitly account for shadowing or small-scale fading. This mismatch induces systematic and time-varying localization error.

\paragraph{Internal State.}
The predictor maintains an internal estimate of planar position and velocity,
\[
\hat{\mathbf{x}}_t \in \mathbb{R}^2, \quad \hat{\mathbf{v}}_t \in \mathbb{R}^2,
\]
which are reset to zero at the beginning of each rollout.

\paragraph{Prediction Step.}
At each time step, the predictor applies a constant-velocity motion model:
\[
\hat{\mathbf{x}}_{t|t-1} = \hat{\mathbf{x}}_{t-1} + \Delta t \, \hat{\mathbf{v}}_{t-1}.
\]
This prediction does not incorporate control inputs or process noise and serves only as a kinematic extrapolation of past estimates.

\paragraph{RSSI-to-Distance Conversion.}
Given an RSSI observation vector $\mathbf{o}_t$, the predictor converts each RSSI measurement into a distance estimate by inverting the log-distance path loss model:
\[
\hat{d}_t^{(i)} = 10^{\frac{P_0 - \text{RSSI}_t^{(i)}}{10 n}}.
\]
This inversion assumes that RSSI measurements are generated solely by path loss and ignores the presence of shadowing and fading. As a result, the inferred distances are biased and noisy, with error statistics that depend on the instantaneous channel realization.

\paragraph{Multilateration via Gauss--Newton.}
The predicted position is refined by solving a nonlinear least-squares multilateration problem of the form
\[
\min_{\mathbf{x} \in \mathbb{R}^2}
\sum_{i=1}^{N}
w_i \left( \|\mathbf{x} - \mathbf{a}_i\|_2 - \hat{d}_t^{(i)} \right)^2,
\quad
w_i = \frac{1}{(\hat{d}_t^{(i)})^2}.
\]
This problem is solved using Gauss--Newton iterations initialized at $\hat{\mathbf{x}}_{t|t-1}$. A small damping term is added for numerical stability, and update steps are clipped to avoid large jumps. The resulting solution is denoted $\mathbf{x}_t^{\text{meas}}$.

\paragraph{Alpha--Beta Filtering.}
The final position and velocity estimates are obtained using an $\alpha$--$\beta$ filter:
\[
\begin{aligned}
\mathbf{e}_t &= \mathbf{x}_t^{\text{meas}} - \hat{\mathbf{x}}_{t|t-1}, \\
\hat{\mathbf{x}}_t &= \hat{\mathbf{x}}_{t|t-1} + \alpha \mathbf{e}_t, \\
\hat{\mathbf{v}}_t &= \hat{\mathbf{v}}_{t-1} + \frac{\beta}{\Delta t} \mathbf{e}_t,
\end{aligned}
\]
with $\alpha = 0.25$ and $\beta = 0.05$. This filter smooths high-frequency measurement noise while allowing the predictor to track gradual motion trends.

\paragraph{Model Mismatch and Error Characteristics.}
Since the predictor assumes a deterministic path loss model and neglects shadowing and small-scale fading, its position estimates exhibit residual errors that are neither independent nor identically distributed. In particular, temporally correlated fading induces time-correlated localization errors whose magnitude and direction vary over time. This mismatch produces nonstationary residual streams, making the predictor suitable for evaluating adaptive conformal uncertainty estimation methods that must respond to evolving error statistics.

\subsubsection{Evaluation Policy}

All experiments use a fixed, non-adaptive random motion policy intended to emulate surveillance-style exploration. At each time step, acceleration inputs are sampled uniformly within bounded limits. Velocity magnitudes are clipped to a maximum of $1.0$ m/s, and accelerations are bounded by $0.5$ m/s$^2$.

The policy is independent of RSSI observations, predictor estimates, and uncertainty regions. The same policy is used across all methods, ensuring that observed differences in coverage and region geometry arise solely from the uncertainty estimation method.

\subsubsection{Expert Weights}
\begin{figure}[h]
    \centering
    \includegraphics[width=0.8\linewidth]{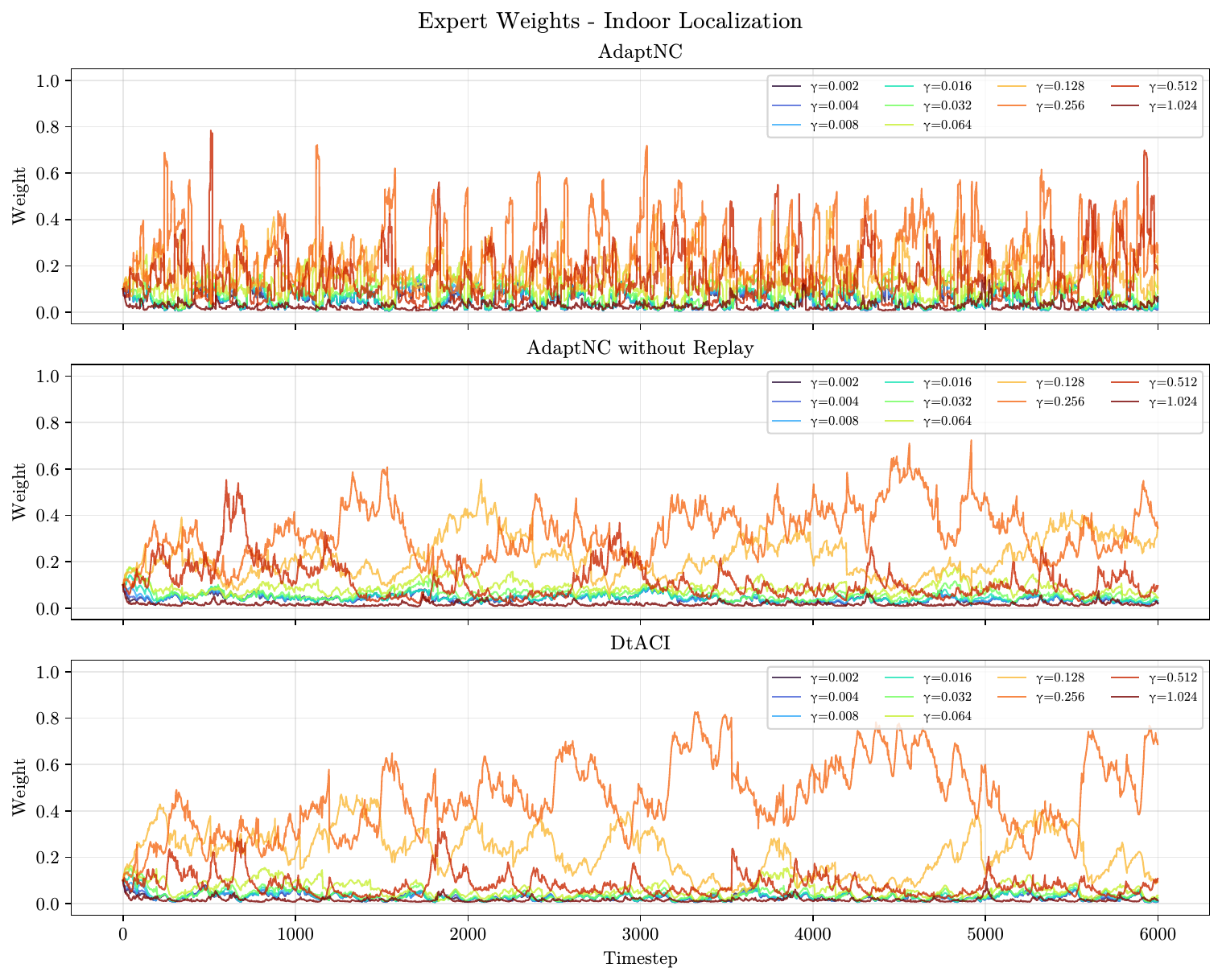}
    \caption{This figure shows the evolution of expert weights in the Indoor Localization setting over the course of the experiment. The pronounced, piecewise variation observed for AdaptNC arises from its adaptive reweighting mechanism, which enables rapid adjustment of expert weights in response to changes in the score function and contributes to maintaining valid coverage. In contrast, DtACI and AdaptNC without replay exhibit slower weight adaptation and consistently select smaller values of $\gamma$ than AdaptNC.}
    \label{fig:appendix-indoor-localization-weights}
\end{figure}\clearpage

\subsection{Social Navigation}\label{appendix:social-navigation}
The simulation environment implements a continuous-space multi-agent system based on the Social Force Model described by Helbing and Molnár (1995).

\subsubsection{Equations of Motion}
Let the state of agent $\alpha$ at time $t$ be defined by its position $\vec{r}_\alpha(t)$, actual velocity $\vec{v}_\alpha(t)$, and preferred velocity $\vec{w}_\alpha(t)$. The dynamics are governed by a nonlinearly dynamics equation driving the preferred velocity, subject to a maximum speed constraint.

The evolution of the preferred velocity $\vec{w}_\alpha$ is given by:
\begin{equation}
    \frac{d\vec{w}_\alpha}{dt} = \frac{\vec{v}_\alpha^{des} - \vec{v}_\alpha}{\tau} + \sum_{\beta \neq \alpha} \vec{F}_{\alpha\beta} + \sum_{B} \vec{F}_{\alpha B} + \vec{\xi}(t)
\end{equation}
where:
\begin{itemize}
    \item $\tau$ is the relaxation time constant.
    \item $\vec{v}_\alpha^{des}$ is the desired velocity vector toward the current goal.
    \item $\vec{F}_{\alpha\beta}$ represents the social repulsive force from agent $\beta$.
    \item $\vec{F}_{\alpha B}$ represents the repulsive force from boundary $B$.
    \item $\vec{\xi}(t)$ is a stochastic fluctuation term (Gaussian noise).
\end{itemize}

The actual velocity $\vec{v}_\alpha(t)$ is derived by clipping the preferred velocity to the agent's maximum speed $v_\alpha^{max}$:
\begin{equation}
    \vec{v}_\alpha(t) = \min\left(1, \frac{v_\alpha^{max}}{\|\vec{w}_\alpha(t)\|}\right) \vec{w}_\alpha(t)
\end{equation}

\subsubsection{Force Components}

\paragraph{Goal Attraction}
The driving term attracts the agent toward its goal $\vec{g}_\alpha$ with a desired speed $v_\alpha^0$:
\begin{equation}
    \vec{v}_\alpha^{des} = v_\alpha^0 \frac{\vec{g}_\alpha - \vec{r}_\alpha}{\|\vec{g}_\alpha - \vec{r}_\alpha\|}
\end{equation}

\paragraph{Social Repulsion}
The interaction force $\vec{F}_{\alpha\beta}$ models the repulsion between agents to maintain personal space. The implementation utilizes an elliptical distance metric approximation. The force is defined as:
\begin{equation}
    \vec{F}_{\alpha\beta} = w(\vec{e}_\alpha, \vec{f}_{\alpha\beta}) \cdot \vec{f}_{\alpha\beta}
\end{equation}
where the raw isotropic force $\vec{f}_{\alpha\beta}$ is:
\begin{equation}
    \vec{f}_{\alpha\beta} = \left( \frac{A}{B} \exp\left(-\frac{b}{B}\right) \frac{2d_{\alpha\beta} - s_{ab}}{2b} \right) \vec{n}_{\alpha\beta}
\end{equation}
with terms defined as follows:
\begin{align*}
    d_{\alpha\beta} &= \|\vec{r}_\alpha - \vec{r}_\beta\| \quad \text{(Euclidean distance)} \\
    \vec{n}_{\alpha\beta} &= \frac{\vec{r}_\alpha - \vec{r}_\beta}{d_{\alpha\beta}} \quad \text{(Normalized direction vector)} \\
    s_{ab} &= \|\vec{v}_\beta\| \Delta t_{anticipation} \quad \text{(Anticipated displacement)}
\end{align*}

The semi-minor axis $b$ of the repulsive potential is approximated in the implementation as:
\begin{equation}
    b = \sqrt{\frac{d_{\alpha\beta}^2 + (d_{\alpha\beta} - s_{ab})^2}{2}}
\end{equation}
This differs slightly from the exact equipotential formulation in the original work but preserves the qualitative elliptical avoidance behavior %
.

\paragraph{Boundary Repulsion}
Repulsion from static boundaries (walls) decays exponentially with the distance $d_{\alpha B}$ to the nearest point on the boundary:
\begin{equation}
    \vec{F}_{\alpha B} = w(\vec{e}_\alpha, \vec{f}_{\alpha B}) \cdot \frac{A_{wall}}{R_{wall}} \exp\left(-\frac{d_{\alpha B}}{R_{wall}}\right) \vec{n}_{\alpha B}
\end{equation}
where $\vec{n}_{\alpha B}$ is the normal vector pointing away from the boundary.

\subsubsection{Anisotropic Perception}
To reflect the limited field of view, forces are scaled by an anisotropic weight $w(\vec{e}, \vec{f})$ dependent on the angle $\phi$ between the agent's desired direction $\vec{e}_\alpha$ and the force vector $\vec{f}$:
\begin{equation}
    w(\vec{e}_\alpha, \vec{f}) = 
    \begin{cases} 
        \lambda + (1-\lambda)\frac{1 + \cos\phi}{2} & \text{if } \cos\phi \ge 0 \quad (\text{Frontal}) \\
        c & \text{if } \cos\phi < 0 \quad (\text{Rear})
    \end{cases}
\end{equation}
where $\cos\phi = \frac{\vec{e}_\alpha \cdot \vec{f}}{\|\vec{e}_\alpha\| \|\vec{f}\|}$.

\subsubsection{Model Parameters}
The default parameters used in the simulation are listed in Table \ref{tab:params}.

\begin{table}[h]
    \centering
    \begin{tabular}{l c l}
        \hline
        \textbf{Parameter} & \textbf{Symbol} & \textbf{Value} \\
        \hline
        Relaxation time & $\tau$ & 0.5 s \\
        Desired speed & $v^0$ & $\mathcal{N}(1.34, 0.26)$ m/s \\
        Max speed & $v^{max}$ & $1.3 \times v^0$ m/s \\
        Repulsion strength & $A$ & 5.0 m/s$^2$ \\
        Repulsion range & $B$ & 2.0 m \\
        Anticipation time & $\Delta t_{anticipation}$ & 2.0 s \\
        Anisotropy factor & $\lambda$ & 0.5 \\
        Rear influence & $c$ & 0.5 \\
        Wall repulsion strength & $A_{wall}$ & 10.0 m/s$^2$ \\
        Wall repulsion range & $R_{wall}$ & 0.2 m \\
        \hline
    \end{tabular}
    \caption{Simulation parameters corresponding to the Social Navigation experiment}
    \label{tab:params}
\end{table}

\subsubsection{Trajectory Predictor}

The trajectory predictor is implemented as an Long Short-Term Memory (LSTM)  to capture temporal dependencies in agent motion. The model is implemented using the Flax/JAX framework.

\paragraph{Architecture}
The model processes a sequence of state observations $X = \{x_1, \dots, x_T\}$ where each timestep $x_t \in \mathbb{R}^{34}$ contains the ego-agent's state and the relative states of the $k=7$ nearest neighbors. The specific feature composition is:
\begin{itemize}
    \item \textbf{Ego Features (6)}: Absolute position $(p_x, p_y)$, velocity $(v_x, v_y)$, and goal position $(g_x, g_y)$.
    \item \textbf{Neighbor Features (28)}: Relative position $(\Delta p_x, \Delta p_y)$ and relative velocity $(\Delta v_x, \Delta v_y)$ for each of the 7 nearest neighbors.
\end{itemize}

The network architecture consists of a single LSTM layer with a hidden dimension of $h=128$. The final hidden state $h_T$ is passed through a Multi-Layer Perceptron (MLP) head to predict the 2D position $\hat{y} \in \mathbb{R}^2$ of the agent at the prediction horizon $\tau$:
\begin{align}
    h_{1:T} &= \text{LSTM}(x_{1:T}) \\
    z &= \text{ReLU}(\text{Dense}_{128}(\text{Dropout}(h_T))) \\
    \hat{y} &= \text{Dense}_2(\text{Dropout}(z))
\end{align}
Dropout with a rate of $0.1$ is applied before the dense layers during training to prevent overfitting %
.

\paragraph{Training Objective}
The model is trained to minimize the Mean Squared Error (MSE) between the predicted position $\hat{y}$ and the ground truth position $y$ after $\tau$ simulation steps:
\begin{equation}
    \mathcal{L}(\theta) = \frac{1}{N} \sum_{i=1}^N \| \hat{y}_i - y_i \|^2
\end{equation}
Optimization is performed using the AdamW optimizer. The specific hyperparameters used for training are detailed in Table \ref{tab:predictor_params}.

\begin{table}[h]
    \centering
    \begin{tabular}{l c}
        \hline
        \textbf{Hyperparameter} & \textbf{Value} \\
        \hline
        Hidden Dimension & 128 \\
        Dropout Rate & 0.1 \\
        Sequence Length ($T$) & 10 \\
        Prediction Horizon ($\tau$) & 5 \\
        Learning Rate & $1 \times 10^{-4}$ \\
        Weight Decay & $1 \times 10^{-5}$ \\
        Batch Size & 512 \\
        Training Epochs & 200 \\
        \hline
    \end{tabular}
    \caption{Hyperparameters for the LSTM Trajectory Predictor.}
    \label{tab:predictor_params}
\end{table}

\subsubsection{Expert Weights}
\begin{figure}[h]
    \centering
    \includegraphics[width=0.8\linewidth]{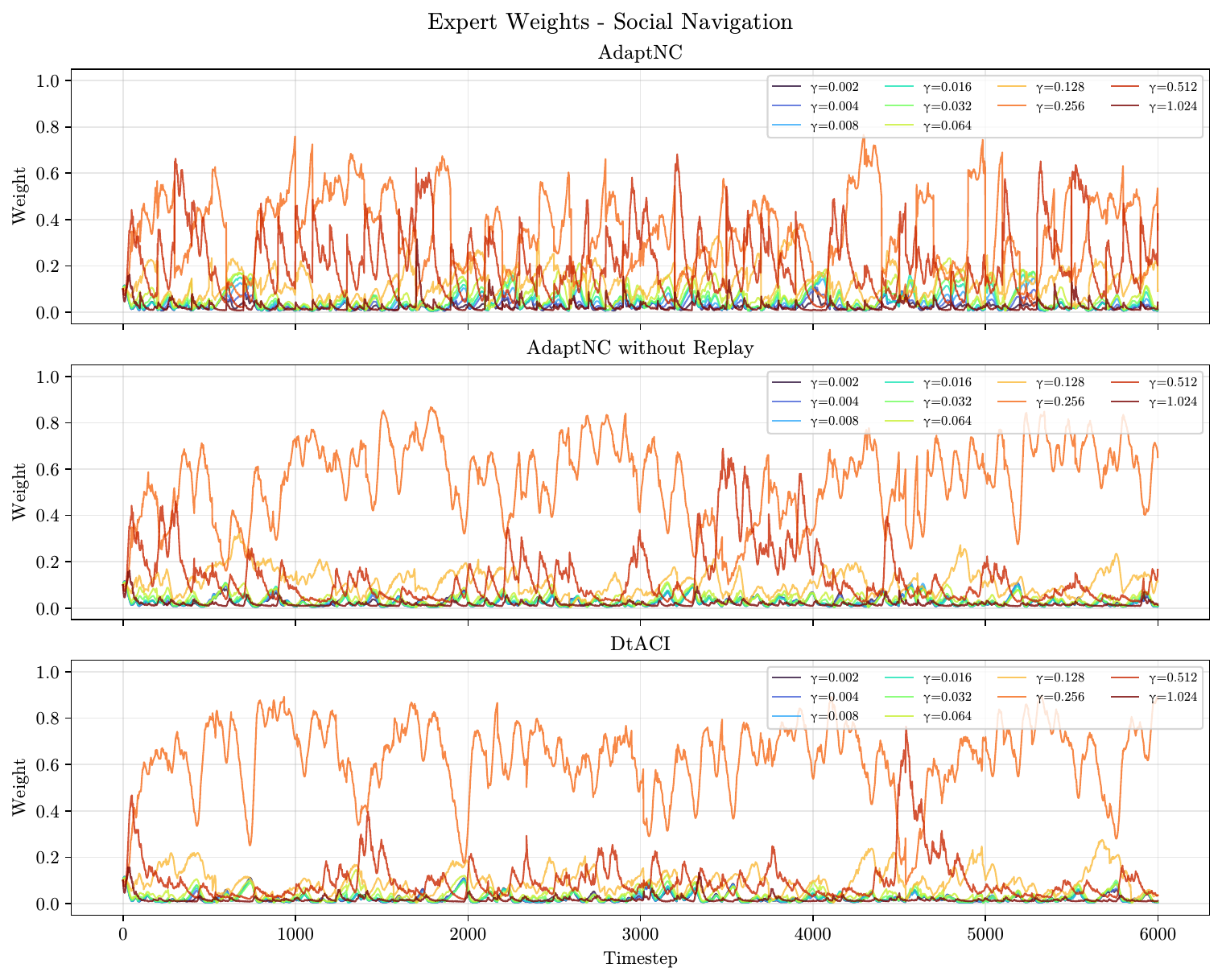}
    \caption{This figure contains the expert weights for the social navigation setting. The plots show the evolution of the weights over the course of the scenario. The stacatto pattern of the AdaptNC setting is a consequence of the adaptive reweighting scheme, which enables rapid adaptation of the expert weights in response to score function changes, which helps AdaptNC maintain coverage. We can see that the DtACI and AdaptNC without Replay methods do not adapt expert weights as quickly, and choose values for $\gamma$ that are lower than those chosen by AdaptNC.}
    \label{fig:appendix-socialnav-weights}
\end{figure}\clearpage

\subsection{Multirotor Tracking}\label{appendix:multirotor-tracking}

This section describes the multirotor environment with actuator degradation, the trajectory predictor used for position forecasting, and the control policy employed to generate training, calibration, and evaluation trajectories.

\subsubsection{Environment with Actuator Degradation}

The multirotor state at time $t$ is: $
\mathbf{s}_t =
[x_t, \dot{x}_t, y_t, \dot{y}_t, z_t, \dot{z}_t,
\psi_t, \dot{\psi}_t, \theta_t, \dot{\theta}_t, \delta_t, \dot{\delta}_t]^\top
\in \mathbb{R}^{12}$, where $(x,y,z)$ denote position, $(\dot{x},\dot{y},\dot{z})$ translational velocity, $\psi$ yaw, $\theta$ pitch, and $\delta$ roll. The system evolves in discrete time with integration step $\Delta t = 0.1$ s using forward Euler integration.

\paragraph{Control Inputs.}
The control input is: $\mathbf{u}_t = [u_{1,t}, u_{2,t}, u_{3,t}, u_{4,t}]^\top$, where $u_1$ corresponds to collective thrust, $u_2$ to pitch torque, $u_3$ to roll torque, and $u_4$ to yaw torque.

\paragraph{Actuator Degradation.}
Each actuator is subject to stochastic degradation modeled by a health variable, $\boldsymbol{\Gamma}_t \in [0,1]^4,$ where $\Gamma_t^{(i)} = 1$ denotes a healthy actuator. The effective control applied to the system is: $\mathbf{u}_t^{\text{eff}} = \boldsymbol{\Gamma}_t \odot \mathbf{u}_t.$ Actuator health evolves according to an additive Wiener process with negative drift:
\[
\boldsymbol{\Gamma}_{t+1}
=
\boldsymbol{\Gamma}_t
- \alpha \Delta t
+ \sigma \mathbf{w}_t,
\quad
\mathbf{w}_t \sim \mathcal{N}(\mathbf{0}, \Delta t I),
\]
with $\alpha = 5 \times 10^{-4}$ and $\sigma = 2.5 \times 10^{-4}$. The degradation process is unobserved by the predictor and policy.
%Health values are clipped to $[0.7, 1.0]$.

\paragraph{Dynamics.}
Translational accelerations are given by :$ \ddot{x}_t = g \theta_t,~ \ddot{y}_t = -g \delta_t,~
\ddot{z}_t = u_{1,t}^{\text{eff}} - g,$
with gravitational constant $g = 9.81$ m/s$^2$.
Rotational accelerations satisfy: $\ddot{\psi}_t = u_{4,t}^{\text{eff}}, ~\ddot{\theta}_t = u_{2,t}^{\text{eff}}, ~\ddot{\delta}_t = u_{3,t}^{\text{eff}}.$ Yaw is wrapped to $[-\pi,\pi]$, and pitch and roll are clipped to $[-0.3, 0.3]$ radians.

\subsubsection{Trajectory Predictor}

To forecast planar motion, we use a hybrid physics learning predictor that combines a simple kinematic prior with a learned residual model.

\paragraph{Physics Prior.}
The prior assumes constant velocity motion in the horizontal plane:
\[
\hat{x}_{t+1}^{\text{phys}} = x_t + \Delta t \dot{x}_t,
\quad
\hat{y}_{t+1}^{\text{phys}} = y_t + \Delta t \dot{y}_t.
\]
This prior ignores rotational coupling and actuator degradation and therefore provides only a coarse approximation of the true dynamics.

\paragraph{Residual Learning.}
Rather than predicting absolute positions, the predictor learns the residual
\[
\mathbf{r}_t =
\begin{bmatrix}
x_{t+1} \\
y_{t+1}
\end{bmatrix}
-
\begin{bmatrix}
\hat{x}_{t+1}^{\text{phys}} \\
\hat{y}_{t+1}^{\text{phys}}
\end{bmatrix}.
\]
The final prediction is obtained as
\[
\hat{\mathbf{p}}_{t+1} =
\begin{bmatrix}
\hat{x}_{t+1}^{\text{phys}} \\
\hat{y}_{t+1}^{\text{phys}}
\end{bmatrix}
+ \hat{\mathbf{r}}_t.
\]

\paragraph{LSTM Architecture and Training.}
Residuals are predicted using a single layer long short term memory network that takes as input a fixed length history
\[
[\mathbf{s}_{t-H+1}, \ldots, \mathbf{s}_t] \in \mathbb{R}^{H \times 12},
\]
with $H = 50$. The LSTM hidden dimension is 128, and the output is a two dimensional residual $(r_x, r_y)$. The network is trained by minimizing mean squared error between predicted and true residuals using supervised data generated from the environment.

\paragraph{Model Mismatch.}
The predictor does not observe actuator health and does not model degradation dynamics. As actuator effectiveness degrades over time, the physics prior becomes increasingly inaccurate. Although the LSTM can partially compensate using temporal context, prediction errors remain time varying and correlated, providing a nonstationary residual process.

\subsubsection{Evaluation Policy}

All trajectories are generated using a fixed model predictive path integral control policy. The policy is used solely to excite the system and is not part of the proposed uncertainty estimation method.

\paragraph{Reference Motion.}
The policy tracks a figure eight trajectory in the horizontal plane:
\[
x_r(t) = A \sin(\omega t), \quad
y_r(t) = A \sin(2 \omega t),
\]
with amplitude $A = 3.0$ and angular frequency $\omega = 0.25$. A constant altitude reference $z_r = 2.0$ m is enforced using a proportional derivative controller on thrust.

\paragraph{MPPI Control.}
At each time step, the policy samples $N = 30$ control sequences over a horizon of $H = 35$ by perturbing a nominal sequence with Gaussian noise. Each sequence is rolled out using the environment dynamics, and a quadratic cost penalizing position tracking error, velocity error, attitude deviation, and control effort is accumulated. Sampled trajectories are weighted using an exponential transformation of cost, and the nominal sequence is updated accordingly. The first control input is applied, and the horizon is receded.

\paragraph{Policy Independence.}
The policy does not observe actuator health, predictor outputs, or uncertainty regions. The same policy and parameters are used across all experiments. Consequently, differences in coverage and region geometry arise solely from the uncertainty estimation method rather than from changes in control behavior.

\subsubsection{Expert Weights}
\begin{figure}[h]
    \centering
    \includegraphics[width=0.8\linewidth]{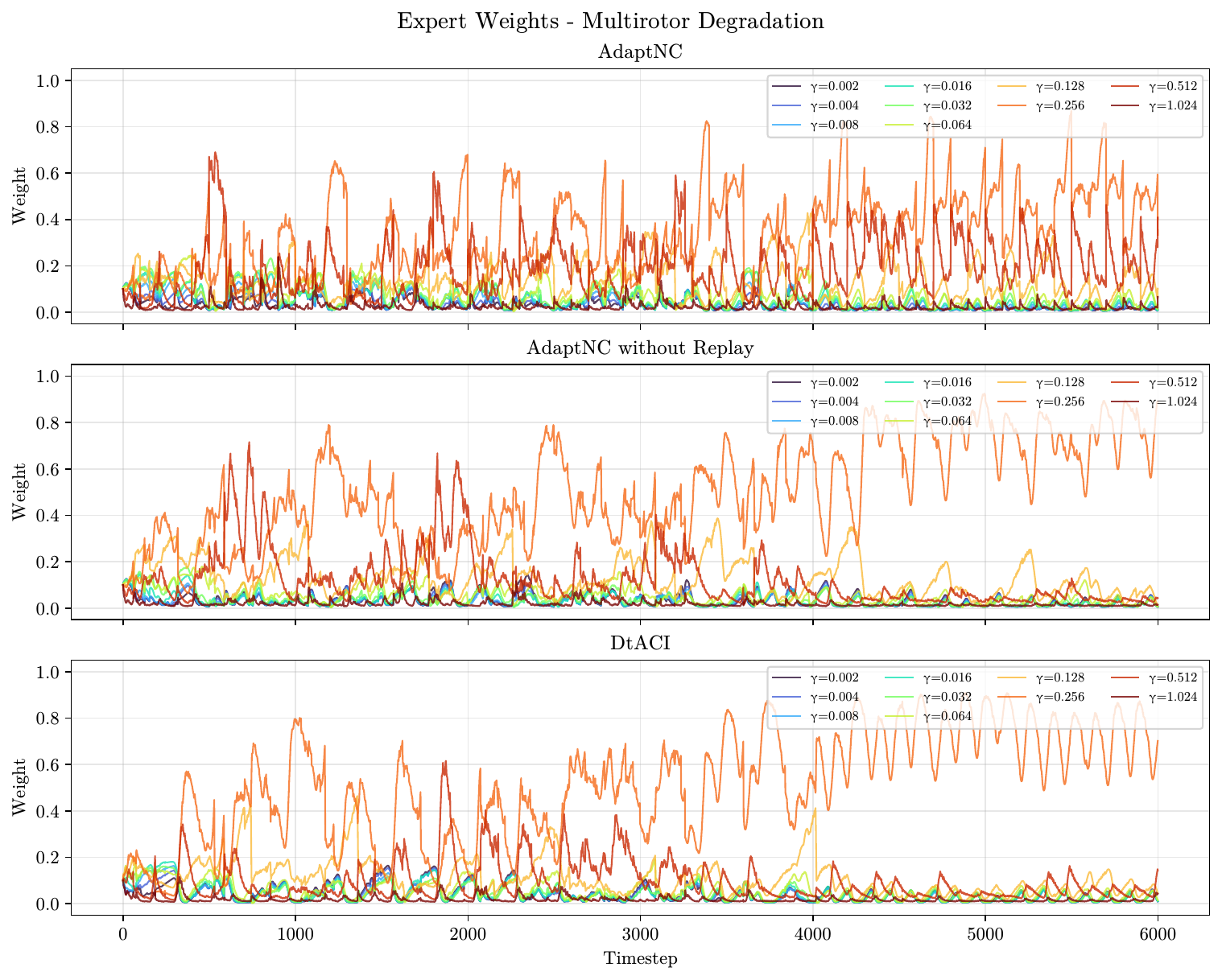}
    \caption{This figure shows the evolution of expert weights in the multirotor navigation setting. The plots illustrate how weights change over time across the different methods. The pronounced oscillatory pattern observed for AdaptNC arises from its adaptive reweighting mechanism, which enables rapid adjustment of expert weights in response to changes in the nonconformity score and contributes to maintaining coverage. In contrast, DtACI and AdaptNC without Replay exhibit slower weight adaptation and assign greater weight to intermediate values of $\gamma$ toward the end of the rollout, resulting in reduced responsiveness compared to AdaptNC.}
    \label{fig:appendix-multirotor-weights}
\end{figure}
\newpage

\subsection{Hyperparameter Sensitivity Analysis}~\label{app:sensitivity_analysis}

\begin{table*}[h]
\centering
\small
\setlength{\tabcolsep}{6pt}
\renewcommand{\arraystretch}{1.1}
\begin{tabular}{lcc}
\toprule

\multicolumn{3}{c}{\textbf{Social Navigation}} \\
\midrule

\multicolumn{3}{l}{\textit{Varying window size (update period = 100)}} \\
\cmidrule(lr){1-3}
Window Size & Global Cov. ($\uparrow$) & Mean Volume ($\downarrow$) \\
\cmidrule(lr){1-3}
35 & 0.9043 $\pm$ 0.0144 & 40.24 $\pm$ 5.66 \\
75 (base) & 0.9223 $\pm$ 0.0090 & 42.00 $\pm$ 5.54 \\
120 & 0.9641 $\pm$ 0.0043 & 52.41 $\pm$ 5.46 \\

\addlinespace

\multicolumn{3}{l}{\textit{Varying update period (window size = 75)}} \\
\cmidrule(lr){1-3}
Update Period & Global Cov. ($\uparrow$) & Mean Volume ($\downarrow$) \\
\cmidrule(lr){1-3}
70 & 0.9609 $\pm$ 0.0032 & 36.45 $\pm$ 3.97 \\
100 (base) & 0.9223 $\pm$ 0.0090 & 42.00 $\pm$ 5.54 \\
130 & 0.8780 $\pm$ 0.0111 & 46.74 $\pm$ 5.93 \\

\midrule

\multicolumn{3}{c}{\textbf{Multirotor Tracking}} \\
\midrule

\multicolumn{3}{l}{\textit{Varying window size (update period = 100)}} \\
\cmidrule(lr){1-3}
Window Size & Global Cov. ($\uparrow$) & Mean Volume ($\downarrow$) \\
\cmidrule(lr){1-3}
35 & 0.9180 $\pm$ 0.0050 & 0.04 $\pm$ 0.01 \\
65 (base) & 0.9264 $\pm$ 0.0055 & 0.04 $\pm$ 0.01 \\
100 & 0.9620 $\pm$ 0.0049 & 0.03 $\pm$ 0.01 \\

\addlinespace

\multicolumn{3}{l}{\textit{Varying update period (window size = 65)}} \\
\cmidrule(lr){1-3}
Update Period & Global Cov. ($\uparrow$) & Mean Volume ($\downarrow$) \\
\cmidrule(lr){1-3}
60 & 0.9685 $\pm$ 0.0027 & 0.03 $\pm$ 0.01 \\
100 (base) & 0.9264 $\pm$ 0.0055 & 0.04 $\pm$ 0.01 \\
140 & 0.8976 $\pm$ 0.0060 & 0.04 $\pm$ 0.01 \\

\midrule

\multicolumn{3}{c}{\textbf{Indoor Localization}} \\
\midrule

\multicolumn{3}{l}{\textit{Varying window size (update period = 20)}} \\
\cmidrule(lr){1-3}
Window Size & Global Cov. ($\uparrow$) & Mean Volume ($\downarrow$) \\
\cmidrule(lr){1-3}
60 & 0.9135 $\pm$ 0.0027 & 26.29 $\pm$ 1.64 \\
100 (base) & 0.9062 $\pm$ 0.0027 & 27.81 $\pm$ 1.75 \\
140 & 0.9004 $\pm$ 0.0032 & 28.67 $\pm$ 1.71 \\

\addlinespace

\multicolumn{3}{l}{\textit{Varying update period (window size = 100)}} \\
\cmidrule(lr){1-3}
Update Period & Global Cov. ($\uparrow$) & Mean Volume ($\downarrow$) \\
\cmidrule(lr){1-3}
15 & 0.8964 $\pm$ 0.0034 & 25.00 $\pm$ 1.71 \\
20 (base) & 0.9062 $\pm$ 0.0027 & 27.81 $\pm$ 1.75 \\
50 & 0.9365 $\pm$ 0.0024 & 35.90 $\pm$ 2.61 \\

\bottomrule
\end{tabular}
\caption{Sensitivity analysis with respect to window size and update period across all environments. Performance remains stable around the chosen base configurations, indicating robustness to moderate hyperparameter variations.}
\label{table:sensitivity}
\end{table*}

\clearpage

    % \newpage
    % \input{checklist.tex}

\end{document}